\documentclass[11pt]{article}
\pdfoutput=1
\usepackage{fullpage}
\usepackage{amsmath,amsfonts,amsthm,amssymb}
\usepackage{url}
\usepackage{graphicx}
\usepackage{caption} 
\usepackage{algorithm}
\usepackage{algorithmic}
\usepackage{bbm}
\usepackage{bm}
\usepackage{float}
\usepackage{framed}
\usepackage{enumerate}
\usepackage{natbib}
\usepackage{color}
\usepackage{multirow}
\usepackage{makecell}
\usepackage{nicefrac}       
\usepackage{subcaption}
\usepackage[colorlinks=true, linkcolor=red, urlcolor=blue, citecolor=blue]{hyperref}
\usepackage{etoolbox}

\newtoggle{arxiv}
\toggletrue{arxiv}

\title{Adaptive Regret for Control of Time-Varying Dynamics}
\date{\today}
\author{  Paula Gradu$^{2, 3}$ \qquad Elad Hazan$^{1, 3}$ \qquad Edgar Minasyan$^{1, 3}$ \\
  $^1$ Department of Computer Science, Princeton University \\
  $^2$ Department of EECS, UC Berkeley \\
  $^3$ Google AI Princeton \\
  \texttt{pgradu@berkeley.edu, \{ehazan,minasyan\}@princeton.edu }  \\
 }

\usepackage{amsmath}
\usepackage{amsfonts,amssymb}
\usepackage{mathtools}
\usepackage{latexsym}
\usepackage{relsize}
\usepackage{enumerate}
\usepackage[capitalise]{cleveref}
\usepackage{xcolor}
\usepackage{dsfont}
\usepackage{float}
\usepackage{graphicx}
\usepackage{algorithm}

\newcommand{\mA}{\mathcal{A}}

\newcommand{\stabgap}{\mathcal{SG}}
\newcommand{\marc}{\mbox{MARC }}
\newcommand{\mara}{\mbox{MARA }}
\newcommand{\dac}{\Pi_{\text{DAC}}}
\newcommand{\drc}{\Pi_{\text{DRC}}}
\newcommand{\lin}{\Pi_{\text{lin}}}
\newcommand{\ldc}{\Pi_{\text{LDC}}}

\newcommand{\R}{\mathbb{R}}

\newcommand{\C}{\mathcal{C}}

\newcommand{\A}{\mathcal{A}}

\newcommand{\opt}{\mathrm{OPT}}

\newcommand{\K}{\ensuremath{\mathcal K}}

\newcommand{\reg}{\mathcal{R}}
\newcommand{\shift}{\mathcal{S}}
\newcommand{\adreg}{Ad\mathcal{R}}
\def\regret{\mbox{{Regret}}}
\def\aregret{\mbox{{AdRegret}}}

\newcommand{\ignore}[1]{}

\newcommand{\eh}[1]{\noindent{\textcolor{blue}{\{{\bf EH:} \em #1\}}}}

\iftoggle{arxiv}{
\theoremstyle{plain}
\newtheorem{theorem}{Theorem}[section]
\newtheorem{lemma}{Lemma}[section]
\newtheorem{corollary}{Corollary}[section]
\newtheorem{definition}{Definition}[section]
\newtheorem{remark}{Remark}[section]

}{}

\newtheorem{claim}{Claim}[section]
\newtheorem{fact}{Fact}[section]
\newtheorem{assumption}{Assumption}[section]

\newtheorem*{theoremaux}{\theoremauxref}
\gdef\theoremauxref{1}

\DeclareMathAlphabet{\mathbfsf}{\encodingdefault}{\sfdefault}{bx}{n}

\DeclareMathOperator*{\argmin}{arg\,min}
\DeclareMathOperator*{\argmax}{arg\,max}

\let\Pr\relax
\DeclareMathOperator{\Pr}{\mathbb{P}}

\renewcommand{\O}{O}

\newcommand{\E}{\mathbb{E}}

\newcommand{\reals}{\mathbb{R}}
\newcommand{\eps}{\varepsilon}

\renewcommand{\leq}{~\le~}
\renewcommand{\geq}{~\ge~}

\let\oldtfrac\tfrac
\renewcommand{\tfrac}[2]{\smash{\oldtfrac{#1}{#2}}}

\let\nablaold\nabla
\renewcommand{\nabla}{\nablaold\mkern-2.5mu}

\begin{document}

\maketitle

\begin{abstract}

We consider the problem of online control of systems with time-varying linear dynamics. This is a general formulation that is motivated by the use of local linearization in control of nonlinear dynamical systems. To state meaningful guarantees over changing environments, we introduce the metric of {\it adaptive regret} to the field of control. This metric, originally studied in online learning, measures performance in terms of regret against the best policy in hindsight on {\it any interval in time}, and thus captures the adaptation of the controller to changing dynamics.

Our main contribution is a novel efficient meta-algorithm: it converts a controller with sublinear regret bounds into one with sublinear {\it adaptive regret} bounds in the setting of time-varying linear dynamical systems.  The main technical innovation is the first adaptive regret bound for the more general framework of online convex optimization with memory. Furthermore, we give a lower bound showing that our attained adaptive regret bound is nearly tight for this general framework.
    
\end{abstract}


\section{Introduction}

Reinforcement learning and control have essentially identical objectives: to maximize long-term reward in a Markov decision process. The focus in control theory is many times on dynamical systems that arise in the real world, motivated by physical applications such as robotics and autonomous vehicles. In these applications the dynamics have succinct descriptions coming from physics equations. These are seldom linear!  Even for simple physical systems such as the inverted pendulum, the dynamics are nonlinear. 
Furthermore, dynamics in the real world often change with time. For example, the dynamics of an UAV flying from source to target may change due to the volatility of the weather conditions (wind, rain, and so forth). 

In terms of {\bf provable methods}, the theory of optimal and robust control has focused on efficient algorithms for linear time invariant (LTI) systems. Nonlinear systems are significantly harder, and in fact NP-hard to control in general \citep{blondel2000survey}. There are several different approaches to deal with nonlinear dynamics, that we detail in the related work section. 
In this paper we consider the approach of iterative linearization by using the first order approximation, that was popularized by planning methods such as iLQR, iLC and iLQG. This allows one to model nonlinear dynamics as {\bf linear time-varying} (LTV) dynamical systems. However, instead of the standard approach of applying planning methods for linear dynamical systems, we build on recent regret minimization algorithms for online control. 

To present our approach to the problem, we first describe the recent literature on online control that differs from classical techniques by measuring performance in terms of regret. We then proceed to show how to borrow concepts from online learning in changing environments to attain meaningful guarantees for control of time-varying systems.

\subsection{Online control of linear dynamical systems} 

A recent advancement in machine learning literature studies the control of dynamical systems in the online learning, or regret minimization, framework.
In the nonstochastic control setting, the controller faces a dynamical system given by 
\begin{equation} \label{eq:system}
x_{t+1} = A_t x_t + B_t u_t + w_t ~.
\end{equation}
Here $A_t,B_t$ describe the system dynamics, $x_t$ is the state, $u_t$ is the control and $w_t$ is the potentially adversarial (nonstochastic) perturbation. Prior literature considers solely the LTI case, where $A_t \equiv A, B_t \equiv B$. The controller chooses the control signal $u_t$, and incurs loss $c_t(x_t, u_t)$, for an adversarially chosen convex cost function $c_t$. 
Since the perturbations and cost functions are arbitrary or chosen adversarially, the best policy is ill-defined a priori. Thus, the performance metric in this model is worst-case regret w.r.t. the best policy in hindsight from a certain policy class $\Pi$. Formally, 
\begin{equation}\label{eq:regret}
\regret_T =    \sum_{t=1}^T c_t (x_t , u_t) - \min_{\pi \in \Pi} \sum_{t=1}^T c_t(x^\pi_t , u^\pi_t)  ~.
\end{equation}
Several benchmark policy classes have been considered in the recent control literature. The simplest to describe is the class of linear state feedback policies, i.e. policies that choose the control as a linear function of the state, $u_t = K x_t$. These policies are known to be optimal for the $\mathcal{H}_2$ control and $\mathcal{H}_\infty$
control formulations for LTI systems.

From this starting point we would like to extend nonstochastic control: {\bf can we prove regret bounds for LTV systems?} How would such bounds even look like? 
To address this question we first consider the field of online learning, where the metric of regret is well studied, and investigate its extension to changing environments.

\subsection{Adaptive regret for online convex optimization}

In the problem of online convex optimization (OCO), a learner iteratively chooses a point in a convex decision set, i.e. $z_t \in \K \subseteq \reals^d$. An adversary then chooses a loss function $f_t: \K \mapsto \reals$. The goal of the learner is to minimize regret, or loss compared to the best fixed decision in hindsight, given as
$$\sum_{t=1}^T f_t(z_t) - \min_{z^\star \in \K } \sum_{t=1}^T f_t(z^\star) ~. $$

The theory of OCO gives rise to efficient online algorithms with sublinear regret, e.g. $O(\sqrt{T})$ over $T$ iterations, implying that on average the algorithm competes with the best fixed decision in hindsight. 

However, the standard regret metric is not suitable for changing environments, where the fixed optimal solution in hindsight is {\it poor}. For example, consider a scenario with $f_t = f$ for the first $T/2$ iterations, and $f_t = g$ for the last $T/2$ iterations. Here, the standard regret metric ensures convergence to the minimum of $f+g$, i.e. the best fixed decision in hindsight which potentially incurs {\it linear} loss\footnote{To see this, take $f(z) = \|z-z^{\star}_f\|^2$ and $g(z) = \|z-z^{\star}_g\|^2$, then the optimum of $f+g$ is the midpoint of $z^{\star}_f$ and $z^{\star}_g$ and over the whole interval suffers cumulative loss linear in $T$.}. Yet, the optimal solution for this scenario is to shift between the minimum of $f$ to that of $g$ midway!

For this reason, the metric of adaptive regret was developed by \cite{hazan2009efficient}. It captures the supremum over all local regrets in any contiguous interval $I$, defined as
\begin{equation}\label{eq:oco_adregret}
\sup_{I = [r, s] \subseteq [1, T]} \left[ \sum_{t = r}^s f_t (z_t) - \min_{z_I^\star  \in \K} \sum_{t = r}^s f_t(z_I^\star ) \right] ~.
\end{equation}
The strength of this definition is that it does {\it not} try to model the changes in the environment. Instead, the responsibility is on the learner to try to compete with the best local predictor $z_I^{\star}$ at all times. For the example above, if an algorithm does not converge to either the optimum of $f$ in $[1, T/2]$ or the optimum of $g$ in $[1, T/2]$, it would suffer \emph{linear} adaptive regret. In general, algorithms that minimize adaptive regret by definition minimize regret as well and additionally are capable of quickly switching between local optima. This metric is thus more appropriate for our investigation of {\it LTV dynamical systems}.

\subsection{Contributions}

In the setting of online control over LTV systems as in \eqref{eq:system}, the adaptive regret metric implies the following: an algorithm that minimizes adaptive regret is capable of competing against different policies from the class $\Pi$ throughout the time horizon $T$. Formally, adaptive regret against a policy class $\Pi$ is given as
\begin{equation}\label{eq:aregret}
\aregret_T = \sup_{I = [r,s] \subseteq [1,T]} \left\{ \sum_{t=r}^s c_t (x_t ,u_t) - \min_{\pi^{\star}_I \in \Pi} \sum_{t=r}^s c_t(x^{\pi^{\star}_I}_t , u^{\pi^{\star}_I}_t) \right\} ~.
\end{equation}
Just like we surveyed for OCO, the main change from standard regret for control is the supremum over all intervals, and the fact that the minimum over the policies is local to the particular interval. 
This ensures that an algorithm with sublinear adaptive regret guarantees enjoys low regret against the best-in-hindsight policy $\pi_I^{\star}$ on any interval $I$. Hence, the algorithm captures any changes in the dynamics of the LTV system by implicitly tracking the local optimal (in $\Pi$) policy.

The main challenge in applying existing adaptive regret methods from online learning to control and reinforcement learning is the long-term effect that actions have.
We first overcome this challenge in the setting of online convex optimization {\it with memory} \citep{anava2015online}. This setting allows one to transfer learning in stateful environments to online learning, following the methodology proposed in \citet{agarwal2019online}. The main contributions of our paper can be summarized as (i) adaptive regret results for online control over LTV systems and (ii) technical contributions in the OCO with memory setting possibly of independent interest.

\paragraph{Adaptive Regret over LTV systems. } We propose an efficient meta-algorithm MARC, Algorithm \ref{alg:adaptive_control_main}, that converts a base controller with standard regret bounds to a control algorithm with provable sublinear adaptive regret guarantees in the setting of linear, time-varying systems.

We apply MARC over recent algorithmic results in nonstochastic control, and obtain an efficient algorithm that attains
$\tilde{O}(\sqrt{\opt})$ adaptive regret against the class of disturbance response control (DRC) policies\footnote{DRC is a very general policy class that encompasses linear dynamical controllers and linear state feedback policies for LTI systems, see Appendix \ref{sec:policyclasses} for details.}, where $\opt$ is the cost of the best policy in hindsight over the entire horizon.

\paragraph{Adaptive Regret in OCO with memory.} Our derivation goes through the framework of OCO with memory, for which we give an efficient adaptive regret algorithm. Specifically, 
our algorithm guarantees $\tilde{O}(\sqrt{\opt})$ adaptive regret for strongly convex functions with memory, where $\opt$ is the best loss in hindsight. This is the first such guarantee for a fundamental setting in prediction.

The aforementioned challenge of obtaining adaptive regret in the setting of OCO with memory is that an online algorithm needs to change its decision slowly to cope with the memory constraint. On the other hand, an online algorithm needs to be agile to quickly adjust to environment changes and minimize adaptive regret. These two requirements are contradictory. We formalize this intuition by proving the following lower bound.

\begin{theorem}[Informal Theorem]
Any algorithm for OCO with memory has 
$\aregret_T = \Omega(\sqrt{T})$, even over strongly convex loss functions. 
\end{theorem}



The lower bound above essentially shows our results to be tight for nonstochastic control algorithms that are based on OCO with memory given that $\opt = O(T)$. However, we note that despite the general lower bound it is still possible to attain $o(\sqrt{T})$ adaptive regret in certain favorable settings. In fact, the positive results in this work, given as first-order adaptive regret bounds, suggest exactly this: Algorithm \ref{alg:adaptive_control_main} suffers adaptive regret much smaller than $\tilde{O}(\sqrt{T})$ when the cost of the best policy in hindsight $\opt = o(T)$ is sublinear. 

\paragraph{Paper Organization.}In subsection \ref{sec:related} we discuss related work. We provide important background and formalize the problem at hand in section \ref{sec:setting}. Section \ref{sec:control} describes the main meta-algorithm and its performance guarantee for online control over changing dynamics. In section \ref{sec:additional_results} we discuss our study of adaptive regret in the OCO with memory setting as these contributions may be of independent interest. Our experimental results are presented in section \ref{sec:experiments}. The appendix includes details, formal theorems, proofs and all else skipped in the main body of the paper for clarity of exposition.

\subsection{Related work}\label{sec:related}

The field of optimal and adaptive control is vast and spans decades of research, see e.g. \citet{Stengel1994OptimalCA,kemin} for survey. 
In terms of nonlinear control, we can divide the literature into several main approaches. The iterative linearization approach takes the local linear approximation via the gradient of the nonlinear dynamics. One can apply techniques from optimal control to solve the resulting changing linear system. Iterative planning methods such as iLQR \citep{iLQR}, iLC \citep{moore2012iterative} and iLQG \citep{todorov2005generalized} fall into this category. Our approach also takes this route.

Another approach is using convex relaxations of the nonlinear dynamics to cope with the hardness of the underlying non-convex optimization. These methods are applied for both $\mathcal{H}_2$ control (see e.g. \citet{majumdar2020recent}), and $\mathcal{H}_{\infty}$ control (see \citet{bansal2017hamilton}) formulations. They are highly effective in some cases, but do not scale well to high dimensional problems. Finally, the nonlinear system can also be linearized via the Koopman operator as detailed in \citet{Koopmanism,clancy}.

In this work we restrict our discussion to online control of changing linear dynamical systems with low \emph{adaptive regret}. 
To the best of our knowledge, this is the first work with adaptive regret bounds shown for time-varying dynamics.

\paragraph{Online convex optimization and adaptive regret.}We make extensive use of techniques from the field of online learning and regret minimization in games \citep{cesa2006prediction,hazan2016introduction}. Most relevant to our work is the literature on adapting to changing environments in online learning, which starts from the works of \citet{herbster1998tracking,bousquet2002tracking}. The notion of adaptive regret was introduced in \citet{hazan2009efficient}, and significantly studied since as a metric for adaptive learning in OCO \citep{adamskiy2016closer,zhang2019adaptive}. An alternative metric for changing systems studied in online learning is called  dynamic regret \cite{zinkevich2003online}. It has been estabilished that dynamic regret is a weaker notion than \emph{strongly} adaptive regret \cite{daniely2015strongly}, in the sense that a sublinear bound on the former implies sublinear dynamic regret, and the reverse is not true \cite{zhang2018dynamic}.

\paragraph{Regret minimization for online control.}In classical control theory, the disturbances are assumed to be i.i.d. Gaussian and the cost functions are known ahead of time. 
In the online LQR setting \citep{abbasi2011regret,dean2018regret,mania2019certainty,cohen2018online}, a fully-observed time-invariant linear dynamic system is driven by i.i.d. Gaussian noise and the learner incurs a cost which is (potentially changing) quadratic in state and input. When the costs are fixed, the optimal policy for this setting is known to be linear  $u_t = K x_t$, where $K$ is the solution to the algebraic Ricatti equation. Several online methods \citep{mania2019certainty,cohen2019learning,cohen2018online} attain $\sqrt{T}$ regret for this setting, and are able to cope with changing loss functions. Regret bounds for partially observed systems were studied in \citet{anima1,anima2,anima3}, with the most general and recent bounds in \citet{simchowitz2020improper}.

\citet{agarwal2019online} consider a significantly more general and  challenging setting, called nonstochastic control, in which the disturbances and cost functions are adversarially chosen, and the cost functions are arbitrary convex costs. In this setting they give an efficient algorithm that attains $\sqrt{T}$ regret. This result was extended to {\it unknown} LTI systems in \citet{hazan2019nonstochastic}, and the partial observability setting in \citet{simchowitz2020improper}. Logarithmic regret for the nonstochastic perturbation setting was obtained in \citet{simchowitz2020making}. For a survey of recent techniques and results see \citet{hazan2021tutorial}.

A roughly concurrent line of work considers minimizing (dynamic) regret against the optimal open-loop control sequence in both LTI and LTV systems. \citet{li2019online} achieve this by leveraging a finite lookahead window while \citet{goel2021regretoptimal} reduce the regret minimization problem to $\mathcal{H}_{\infty}$ control. \citet{zhang2021strongly} follow up our work to devise methods with \emph{strongly} adaptive regret guarantees however these regret bounds, as opposed to ours, are not first-order.


\section{Problem Setting and Preliminaries}\label{sec:setting}
\paragraph{Notation.} Throughout this work we use $[n] = [1,2,..., n]$ as a shorthand, $\| \cdot \|$ is used for Euclidean and spectral norms, $O(\cdot)$ hides absolute constants, $\tilde{O}(\cdot)$ hides terms poly-logarithmic in $T$.

\subsection{Online LTV Control}\label{sec:setting_control}
A time-varying linear (LTV) dynamical system is given by the following dynamics equation,
\begin{equation*}
\forall t \in [T], \quad x_{t+1} = A_t x_t + B_t u_t + w_t ,
\end{equation*}
where $x_t \in \R^{d_x}$ is the (observable) system state, $u_t \in \R^{d_u}$ is the control, and 
$(A_t, B_t)$ are the system matrices with $A_t \in \R^{d_x \times d_x}$, $B_t \in \R^{d_x \times d_u}$, $w_t \in \R^{d_x}$ is the disturbance. 
In our work, we allow $w_t$ to be adversarially chosen. This is the key assumption in the {\it nonstochastic} control literature. The additional generality of adversarial perturbations allows the disturbance to model slight deviations from linearity along with inherent noise.

\ignore{
\eh{TEMP: 
It is shown in \citep{hazan2019nonstochastic} that for an unknown system $(A, B)$ and a known $(\kappa, \gamma)$ strongly stable controller $K$, a disturbance-action controller can achieve sublinear regret against all such controllers parametrized by $(K', M)$ where $K'$ is $(\kappa, \gamma)$ strongly stable.
\begin{definition}
A disturbance-action controller with parameters $(K, M)$ where $M = [M^0, M^1, \ldots, M^{H-1}]$ outputs control $u_t$ at state $x_t$,
$$ u_t = Kx_t + \sum_{i=1}^H M_{i-1} w_{t-i}. $$
\end{definition}}
}

We consider the setting of {\it known} systems, i.e. after taking an action $u_t$ the controller observes the next state $x_{t+1}$ as well as the current system matrices $(A_t, B_t)$. This allows the controller to compute the disturbance $w_t = x_{t+1} - A_t x_t - B_t u_t$, so the knowledge of $x_{t+1}$ and $w_t$ is interchangeable. 
A control algorithm $\mathcal{C}$ chooses an action $u_t = \mathcal{C}(x_1, \dots, x_t)$ based on previous information. It then suffers loss $c_t(x_t,u_t)$ and observes the cost function $c_t$. We remark that an adaptive adversary chooses all $(A_t, B_t), w_t, c_t$. We make the following basic assumptions common in the nonstochastic control literature.
\begin{assumption}\label{assmp:noise_bound}
The disturbances are bounded in norm, $\max_t \| w_t \| \le W$.
\end{assumption}
\begin{assumption}\footnote{Dynamics satisfying this assumption are called sequentially stable which we relax to sequential stabilizability in Appendix \ref{sec:policyclasses}.}
There exist $C, C_B \ge 1$ and $\rho \in (0, 1)$ such that for all $t$ and $H \in [1, t)$, 
\begin{align*}
    \Phi_t^H = \prod_{i=t}^{t-H+1} A_i, \, \left\| \Phi_t^H \right\|_{\text{op}} \le C \cdot \rho^H, \, \| B_t \|_{\text{op}} \le C_B ~.
\end{align*}
\label{assmp:seq_stable}
\end{assumption}
\begin{assumption}\label{assmp:control_cost}
The cost functions $c_t : \R^{d_x} \times \R^{d_u} \to \R$ are general convex functions that satisfy the conditions $0 \le c_t(x, y) \le 1$ and $\| c_t(x, y) \| \leq L_c \max\{1, \|x\|+\|u\|\}$ for some $L_c > 0$.
\end{assumption}
The standard performance metric of controller $\mathcal{C}$ over horizon $T$ is regret with respect to a class of policies $\Pi$ as defined in \eqref{eq:regret} denoted $\regret_T(\mathcal{C})$. We instead minimize for the {\it adaptive regret} metric of $\mathcal{C}$ with respect to $\Pi$ as defined in \eqref{eq:aregret}, and denote it $\aregret_T(\mathcal{C})$. In case the control algorithm is randomized, we take the expectation of the metric over the randomness in the algorithm.

The choice of the policy class $\Pi$ is essential for the performance of a control algorithm. One target class of policies we compare against in this paper is disturbance action controllers (DAC), 
whose control is a linear function of finite past disturbances $u_t = \sum_{i=1}^H M_i w_{t-i}$, for some history-length parameter $H$. This comparator class is known to approximate to arbitrarily high precision the state-of-the-art in LTI control: linear dynamical controllers (LDC). This choice is a consequence of recent advances in convex relaxation for control  \citep{agarwal2019online,agarwal2019logarithmic,hazan2019nonstochastic,simchowitz2020improper}. The convex relaxation in all these advances is accompanied with a reduction of online control to the setting of online convex optimization (OCO) with memory.

\paragraph{Reduction to OCO with memory.}
The intuition behind this approach is that even though actions have long-term effect in control, their effect is dissipating geometrically fast in time. Thus,  actions and states that occurred far in the past have only marginal effect on the dynamics as a whole. The formal statement for this intuition is given in Definition \ref{def:membounded}, a generalization of Definition 2.1 from \citet{agarwal2020boosting}.

Before stating the formal definition, we first describe the notion of an {\it action set sequence}. For a fixed horizon $T$, let $\mathcal{U}_t \subset \R^{d_u}$ be the constraint set for action $u_t$ for each $t \in [T]$. Denote the action set sequence $\mathcal{U}_{1:T} = \{\mathcal{U}_1, \dots, \mathcal{U}_T \}$ and use $u_{1:T} \in \mathcal{U}_{1:T}$ to indicate $u_t \in \mathcal{U}_t$ for all $t \in [T]$. We remark that $\mathcal{U}_t$ potentially depends on the system dynamics up to time $t$ and the action set sequence $\mathcal{U}_{1:T}$ depends on the family of control algorithms used, but {\it not} on the particular individual instance.

\begin{definition}\label{def:membounded}
The action set sequence $\mathcal{U}_{1:T}$ is said to have $(H, \eps)$-bounded memory if for all fixed arbitrary $u_{1:T} \in \mathcal{U}_{1:T}$ and all $t \in [T]$,
\begin{equation*}
    | c_t(x_t, u_t) - c_t(\hat{x}_t, u_t) | \leq \eps,
\end{equation*}
where we define $\hat{x}_t$ to be the proxy state with memory $H$ for the sequence of actions $u_{1:T}$: for $t > H$, $\hat{x}_t$ is the state reached by the system if we artificially set $x_{t-H} = 0$ and simulate the dynamics \eqref{eq:system} with the actions $u_{t-H}, \dots, u_{t-1}$.
\end{definition}

Suppose an action set sequence $\mathcal{U}_{1:T}$ has $(H, \eps)$-bounded memory. Then, given that $\hat{x}_t = \hat{x}_t(u_{t-H}, \dots, u_{t-1})$, the performance guarantees of a proxy cost function $f_t(u_{t-H:t}) = c_t(\hat{x}_t, u_t)$ imply guarantees for the control setting. Furthermore, regret minimization of $f_t(u_{t-H:t})$ can be done in the setting of OCO with memory. Finally, we state the necessary properties for a controller $\mathcal{C}$ to be considered a base control algorithm.

\begin{definition}\label{def:base_controller}
A control algorithm $\mathcal{C}$ with an action set sequence $\mathcal{U}_{1:T}$ is called a base controller if:
    \begin{enumerate}[(i)]
        \item $\mathcal{U}_{1:T}$ has $(H, \eps)$-bounded memory.
        \item for all $t \in [T]$, the proxy loss $f_t(u_{t-H:t}) = c_t(\hat{x}_t, u_t)$ is coordinate-wise $L$-Lipschitz.
    \end{enumerate}
\end{definition}

Note that the properties of a base controller concern the control algorithm setup not the controller instance itself. The definition of a base controller serves simply as an abstraction:
in Appendix \ref{sec:apply_control} we show it is not vacuous given that all previous control algorithms in the nonstochastic control literature satisfy the conditions \citep{agarwal2019online, simchowitz2020making}.

\ignore{
\begin{claim}\label{claim:appl}
Let $\mathcal{C}_{\text{GFC}}$ be the GFC control algorithm \citep{simchowitz2020improper}, with $\mathcal{U}_{1:T}$ action set sequence and policy class $\Pi$ of strongly stable linear controllers. This setup satisfies the first two properties of Assumption \ref{assmp:control} with $H = \O(\log(1/\eps))$ and $L = \O(\sqrt{H})$. Furthermore, on a fixed LDS, i.e. $(A_t, B_t) = (A, B), \forall t \in [T]$, the controller $\mathcal{C}_{\text{GFC}}$ attains policy regret $\reg_{\C_{\text{GFC}}}(T) = \Theta(\text{poly}(\log T))$ with respect to $\Pi$ and action shift $\shift_{\C_{\text{GFC}}}(T) = \Theta(\text{poly}(\log T))$.
\end{claim}
}

\subsection{Online Convex Optimization with Memory}\label{sec:setting_oco}
We address the problem of online control of a linear dynamical system from the setting of online convex optimization (OCO, see e.g. \citet{hazan2016introduction}), and specifically OCO with memory.  
In many environments the decisions of the player affect the long-term future, such is the case in the problems of control and reinforcement learning. Hence, in the setting with memory the adversary reveals the loss function $f_t : \K^{H+1} \mapsto \reals$ that applies to the past $H+1$ decisions of the player, and the player suffers loss $f_t(z_{t-H:t})$ where $z_{i:j} = \{z_i, \dots, z_j\}$ with $i < j$. Define the surrogate loss $\tilde{f}_t : \mathcal{K} \mapsto \R$ to be the function with all $H+1$ arguments equal, i.e. $\tilde{f}_t(z) = f_t(z, \dots, z)$ (this reduces to the standard setting with $H=0$). The regret in this setting is defined with respect to the best surrogate loss in hindsight as follows, where we denote $z_i = z_1$ for all $i < 1$,
\begin{equation*}
    \regret_T = \sum_{t=1}^T f_t(z_{t-H:t}) - \min_{z \in \mathcal{K}} \sum_{t=1}^T \tilde{f}_t(z) ~.
\end{equation*}

\paragraph{Adaptive regret.}
It has been observed in the literature that the standard regret benchmark promotes convergence to the static optimum of the entire learning sequence, and in certain cases hinders adaptivity to changing environments. To remedy this situation, \citet{hazan2009efficient} proposed the more refined adaptive regret metric as defined in \eqref{eq:oco_adregret}, which promotes adapting to environment changes. In the setting of OCO with memory, the notion of adaptive regret is defined analogously, i.e. the supremum of the local standard regret over all contiguous intervals,
\begin{equation*}
\aregret_{T}(\mathcal{A}) = \sup_{I=[r,s] \subseteq [T]} \left[ \sum_{t=r}^s f_t(z_{t-H:t}) - \min_{z^\star_I \in \mathcal{K}} \sum_{t=r}^s \tilde{f}_t(z^\star_I) \right] ~.
\end{equation*}
The OCO setting with memory, as outlined in \citet{anava2015online}, reduces to the standard setting by assuming a Lipschitz condition on $f$ and relating $f$ to $\tilde{f}$. To quantify this relation, it is crucial to keep track of the movement between the consecutive actions by the learner. Henceforth, we define the notion of {\it action shift}, a metric for the stability of the algorithm, as the overall shifting in distance of consecutive actions by $\A$, 
\begin{equation}\label{eq:actionshift}
     \shift_{T}(\A)  = \sup_{z_1, \dots, z_T \leftarrow \A} \left[ \sum_{t=1}^{T-1} \| z_{t+1} - z_t \| \right] ~.
\end{equation}
Action shift is necessary to quantify an algorithm's prediction stability: the predictions are stable if the action shift is small and this is crucial in the setting with memory. On the other hand, adaptive regret encourages an algorithm to move quickly to adapt to environment changes, thus compromises the prediction stability by driving the action shift to be large.


\section{Online Control of Time-Varying Dynamics}\label{sec:control}

For the setting of online control described in Section \ref{sec:setting_control}, we devise MARC (Algorithm \ref{alg:adaptive_control_main}): a meta-algorithm that takes a base controller $\mathcal{C}$ and ``transforms'' its standard regret guarantees into adaptive regret bounds. It does so by maintaining $N=T$ copies of the base controller ($\mathcal{C}_1, \ldots, \mathcal{C}_N$), with the restriction that each $\mathcal{C}_i$ plays $u_t^i=0$ for $t<i$ and only starts running $\mathcal{C}$ at round $t=i$. At each round $t$, MARC chooses the action $u_t=u_t^i$ given by $\mathcal{C}_i$ with some probability that reflects $\mathcal{C}_i$'s performance so far. As long as $\mathcal{C}$ is a base controller according to Definition \ref{def:base_controller}, we can transfer our general results on adaptive regret for online convex optimization with memory from Section \ref{sec:additional_results} to the control setting, yielding Theorem \ref{thm:meta_controller_main} given below. 

\paragraph{Black-box use of $\mathcal{C}$.} Since the system is known to the meta-controller, each controller $\mathcal{C}_i$ can construct a simulated environment with its own actions $u_t^i$ and identical system matrices and disturbances. In particular, once the meta-controller observes the new state $x_{t+1}$, it computes the corresponding disturbance $w_t = x_{t+1} - A_t x_t - B_t u_t$ and feeds it to the base controllers along with the system matrices $(A_t, B_t)$. Afterwards, each base controller $\mathcal{C}_i$ simulates the system environment with its own action, i.e. $x_{t+1}^i = A_t x_t^i + B_t u_t^i + w_t$. Such behavior allows for black-box use of results for the base controllers since each acts separately in response to the same dynamics.

\begin{algorithm}
\caption{Meta Adaptive Regret Controller (MARC) \label{alg:adaptive_control_main}} 
\begin{algorithmic}[1]
\STATE \textbf{Input}: horizon $T$, action set sequence $\mathcal{U}_{1:T}$, $N=T$ controllers $\mathcal{C}_1, \dots, \mathcal{C}_N$, parameters $\eta, \sigma$
\STATE \textbf{Setup:} assign $w_1^i = 1$ and feedback $\mathcal{F}^i_1 = \{x^i_1 = 0\}, \forall i \in [N]$, denote $W_t = \sum_{i=1}^N w_t^i$
\FOR{$t = 1, ..., T$}
\STATE compute each action $u_t^i$ by $\mathcal{C}_i$ given $\mathcal{F}_t^i$
\IF{$t = 1$}
    \STATE choose $i_t = i$ w.p. $p_t^i = w_t^i / W_t$ for all $i \in [N]$
\ELSE
    \STATE keep $i_t = i_{t-1}$ w.p. $w_t^{i_{t-1}} / w_{t-1}^{i_{t-1}}$, o.w. choose $i_t = i$ w.p. $p_t^i = w_t^i / W_t$ for all $i \in [N]$
\ENDIF
\STATE choose action $u_t = u_t^{i_t}$, observe $c_t(\cdot, \cdot)$, suffer cost $c_t(x_t, u_t)$
\STATE observe new state $x_{t+1}$, compute $w_t$ disturbance, obtain $x_{t+1}^i$ given $u_t^i, w_t$ for all $i \in [N]$
\STATE let $f_t(u_{t-H:t}) = c_t(\hat{x}_t, u_t)$ be proxy cost, $\tilde{f}_t(u) = f_t(u, \dots, u)$ be surrogate proxy cost
\STATE compute $\overline{w}_{t+1}^i = w_t^i e^{-\eta \tilde{f}_t(u_t^i)}$ and $w_{t+1}^i = (1-\sigma) \overline{w}_{t+1}^i + \sigma \overline{W}_{t+1} / N$ for all $i \in [N]$
\STATE update $\mathcal{F}^i_{t+1} = \mathcal{F}^i_{t} \cup \{x^i_{t+1}, u^i_t, c_t\}$ for all $i \in [N]$
\ENDFOR
\end{algorithmic}
\end{algorithm}


\paragraph{Efficient implementation.} We remark that Algorithm \ref{alg:adaptive_control_main} is not computationally efficient relative to a base controller $\mathcal{C}$: it has the computational complexity of $T$ such controllers. Yet, our algorithm can be implemented in an efficient manner by keeping track of and updating only $O(\log T)$ active controllers. The inactive ones are represented by the stationary $u_t=0$ controller. This efficient version incurs only a $O(\log T)$ extra multiplicative adaptive regret factor relative to Theorem \ref{thm:meta_controller_main} and only $O(\log T)$ computational overhead relative to the base controller $\mathcal{C}$. For the sake of clarity and brevity, we present Algorithm \ref{alg:adaptive_control_main} without this component and present the efficient implementation formally in Appendix \ref{sec:efficient}.

\begin{theorem}\label{thm:meta_controller_main}
Let $\mathcal{C}$ be a base control algorithm by Definition \ref{def:base_controller} with $\epsilon = T^{-1}$ and denote $\opt = \min_{\pi \in \Pi} \sum_{t=1}^T c_t(x_t^{\pi}, u_t^{\pi})$. With the parameter choices of $\eta = \tilde{\mathcal{O}}(\shift_T(\mathcal{C}) \sqrt{\opt})^{-1}$ and $\sigma=T^{-1}$, Algorithm \ref{alg:adaptive_control_main} (MARC) achieves the following adaptive regret against the class of policies $\Pi$:
\begin{equation}\label{eq:controlopt_main}
    %
        \E\left[\aregret_{T}(\text{MARC})\right] \leq \tilde{\mathcal{O}}\left( \regret_{T}(\mathcal{C})  +  \shift_{T}(\mathcal{C}) \sqrt{\opt}   \right),
\end{equation}
where $\regret_T(\mathcal{C})$ is the regret $\mathcal{C}$ attains w.r.t. $\Pi$, and $\shift_T(\mathcal{C})$ is its action shift.
\end{theorem}

\begin{remark}
In \eqref{eq:controlopt_main}, the component $\regret_{T}(\mathcal{C})$ is the supremum of $\regret_{|I|}(\mathcal{C})$ for each $I$. However, the components $\shift_T(\mathcal{C})$ and $\opt$ are global even for local regret bound of $I$. This is due to the parameter choice of $\eta$ which can be done with no prior knowledge of $\opt$ (see Appendix \ref{sec:param_choice}). 
\end{remark}

\begin{remark}
The Theorem \ref{thm:meta_controller_main} bound (\ref{eq:controlopt_main}) hides factors of quantities $H, L$ in the $\tilde{\mathcal{O}}$ notation. This is justified because in the relevant nonstochastic control literature these quantities have been shown to be poly-logarithmic in $T$ (see Appendix \ref{sec:apply_control}).
\end{remark}

Theorem \ref{thm:meta_controller_main} guarantee shows that Algorithm \ref{alg:adaptive_control_main} converts a standard regret bound of a base controller into an adaptive one with an extra additive term as in \eqref{eq:controlopt_main}. Hence, of utmost interest are algorithms from the literature that achieve logarithmic regret (and action shift) \citet{agarwal2019logarithmic, simchowitz2020making} (as illustrated in Corollary \ref{cor:drc-ons}). \\

\begin{proof}[Proof Sketch]
First, notice that each instance of the base controller acts on the same action set sequence, hence the meta-controller's actions also lie within this set sequence. This means that MARC automatically inherits the properties from Definition \ref{def:base_controller}.

Property (i): the meta-controller satisfies the bounded memory definition which ensures we can translate bounds in terms of proxy loss to bounds in terms of control cost (and vice versa). We relate the control cost of the agent to its proxy loss as follows (translating the proxy loss of the best policy to its control cost is analogous): for every $I=[r,s]$, by Definition \ref{def:membounded} and since $\eps = T^{-1}$, we have
\begin{align*}
    \sum_{t=r}^s \E[c_\mathrm{t}(x_\mathrm{t}, u_\mathrm{t})] \leq \sum_{t=r}^s \E[c_\mathrm{t}(\hat{x}_\mathrm{t}, u_\mathrm{t})] + |\mathrm{I}| \eps 
     \leq \sum_{t=r}^s \E[f_\mathrm{t}(u_{\mathrm{t-H:t}})] + 1 
\end{align*}

Property (ii): the Lipschitz property allows us to use our adaptive regret results for functions with memory to the proxy losses $f_t$ (see Theorem \ref{thm:adaoco_informal}). This ensures that on $I=[r,s]$ the proxy loss local regret w.r.t. $\mathcal{C}_r$ is $\tilde{O}(H^2 L \shift_T(\mathcal{C}) \sqrt{\opt})$. Furthermore, the base controller $\mathcal{C}_r$ suffers regret $\regret_{|I|}(\mathcal{C})$ against the best policy in hindsight from $\Pi$. Putting it all together yields \eqref{eq:controlopt_main}.
\end{proof}

\paragraph{Application to DRC-ONS.}The statement in Theorem \ref{thm:meta_controller_main} indicates a general black-box result: a low regret control algorithm with certain regularity assumptions can be turned into one with low \emph{adaptive} regret. We showcase the use of our meta-algorithm MARC on the DRC-ONS \citep{simchowitz2020making} algorithm.
\begin{corollary}[MARC-DRC-ONS]\label{cor:drc-ons}
Let $\mathcal{C}$ be the DRC-ONS algorithm from \citep{simchowitz2020making}. Assume the cost functions $c_t$ are strongly convex. The MARC-DRC-ONS algorithm that performs MARC over $\mathcal{C}$ enjoys the following adaptive regret guarantee w.r.t. the DRC policy class $\Pi_{\mathrm{drc}}$\footnote{See Appendix \ref{sec:policyclasses} for definition.},
\begin{equation}
    \E\left[\aregret_{T}(\text{MARC-DRC-ONS})\right] \leq \widetilde{\mathcal{O}}\left( \sqrt{\opt}   \right) ~.
\end{equation}
\end{corollary}
The result above holds for any LTV system that satisfies Assumptions \ref{assmp:noise_bound}, \ref{assmp:seq_stable}, \ref{assmp:control_cost}. Note that the adaptive regret bound $\widetilde{\mathcal{O}}(\sqrt{\opt})$ can be considerably smaller than $\sqrt{T}$ if the LTV system (along with the disturbances and costs) is favorable. More importantly, the bound, and more specifically $\opt$, depends directly on the policy benchmark we are competing against, $\Pi_{\mathrm{drc}}$ in this case. Hence, algorithms that enjoy low regret against larger policy classes automatically enjoy better adaptive regret bounds via MARC. The formal statement and proof of Corollary \ref{cor:drc-ons} and application to other control algorithms can be found in Appendix \ref{sec:control-full}. 

\section{Additional Results in OCO with Memory}\label{sec:additional_results}

The online control results presented in Section \ref{sec:control} hinge on the more general results derived for the setting of OCO with memory which may be of independent interest. We informally state and discuss our results for this setting in this section and defer the formal complete treatment to Appendix \ref{sec:ocom_adareg}. Obtaining adaptive regret in this setting is tricky due to the need to balance agility to environment changes and stability of predictions. Our meta-algorithm effectively does this by combining the FLH approach \citep{hazan2009efficient} with a shrinking technique \citep{geulen2010regret} to maintain prediction stability. This allows for novel adaptive regret results over strongly convex loss functions with memory.

\begin{theorem}\label{thm:adaoco_informal}
Suppose the loss functions $f_t$ have $H$ memory and are coordinate-wise $L$-Lipschitz. Furthermore, assume their surrogate functions $\tilde{f}_t$ are strongly convex. Then, our meta-algorithm with OGD as the base algorithm attains the following adaptive regret bound,
\begin{equation}\label{eq:optregret_main}
  \E \left[\aregret_T\right] \leq \tilde{\mathcal{O}}(H^2 L \cdot \sqrt{OPT}) ~.
\end{equation}
\end{theorem}

We also provide an efficient variant of the algorithm with only an additional $O(\log T)$ multiplicative factor on the bound in \eqref{eq:optregret_main}. The idea behind the efficient implementation is to avoid the redundancy of maintaining all the experts and instead only consider a sparse working set: update the weights of these {\it active} experts separately, and handle the {\it inactive} expert weights in a collective way. This ensures that we only incur $O(\log T)$ computational overhead over efficient non-adaptive methods such as OGD.

Henceforth, we show that it is impossible to achieve $o(\sqrt{T})$ adaptive regret for functions with memory even under the assumption of strong convexity. Recall that, given Lipschitz loss and strongly convex surrogate loss, $O(\log T)$ logarithmic {\it standard} regret is attained for functions with memory \citep{anava2015online}. Furthermore, for strongly convex losses with {\it no} memory ($H = 0$) adaptive (poly-)logarithmic regret is also attained via the FLH method \citep{hazan2009efficient}. However, these two cannot be combined together to attain adaptive regret bounds less than $\sqrt{T}$.

\begin{theorem}\label{thm:lower_informal}
Assume the losses $f_t$ have $H>0$ memory and their surrogate losses are strongly convex. For any online algorithm $\mathcal{A}$, there exists a sequence of loss functions s.t.
\begin{equation}\label{eq:lower}
    \aregret_T(\A) = \Omega(\sqrt{T}) ~.
\end{equation}
\end{theorem}

The key insight behind this lower bound is that in the standard OCO setting with {\it no} memory, adaptive regret guarantees necessarily mean large action shift. More specifically, we show that any algorithm with $R = o(T)$ adaptive regret needs to have large action shift $\Omega(T/R)$. We construct a loss function sequence that drives the algorithm to $\pm 1$ alternatively given the adaptive regret bound and obtain the lower bound on action shift. This tradeoff deters loss functions with memory to exhibit low adaptive regret despite the strong convexity assumption.


\section{Experimental Results} \label{sec:experiments}

To show the applicability of our algorithm in more realistic (and harder) scenarios, we consider the control of a nonlinear system via iterative linearization as detailed in the introduction. 
We experiment on the inverted pendulum environment, a commonly used benchmark consisting of a nonlinear and unstable system, popularized by OpenAI Gym \citep{brockman2016openai}.

For this system, the state consists of the deviation angle $\theta$ and the rotational velocity $\dot{\theta}$, while the action corresponds to the applied torque $\ddot{\theta}$. The objective is to balance the inverted pendulum by applying torque that will stabilize it in a vertically upright position. At each timestep, the learner incurs a cost of $\theta^2 + 0.1 \cdot \dot{\theta}^2 + 0.001 \cdot \ddot{\theta}^2$, where $\theta$ is normalized between $-\pi$ and $\pi$. We first experiment with the original noiseless system. We conduct a second experiment in which we introduce a sinusoidal shock in the middle of the run, i.e. for timesteps $ t\in [T/3, 2T/3]$ an adversarial perturbation of $0.3*\sin(t/2\pi)$ is added to the state.



\begin{figure}[H]
\centering
\includegraphics[width=66.2mm]{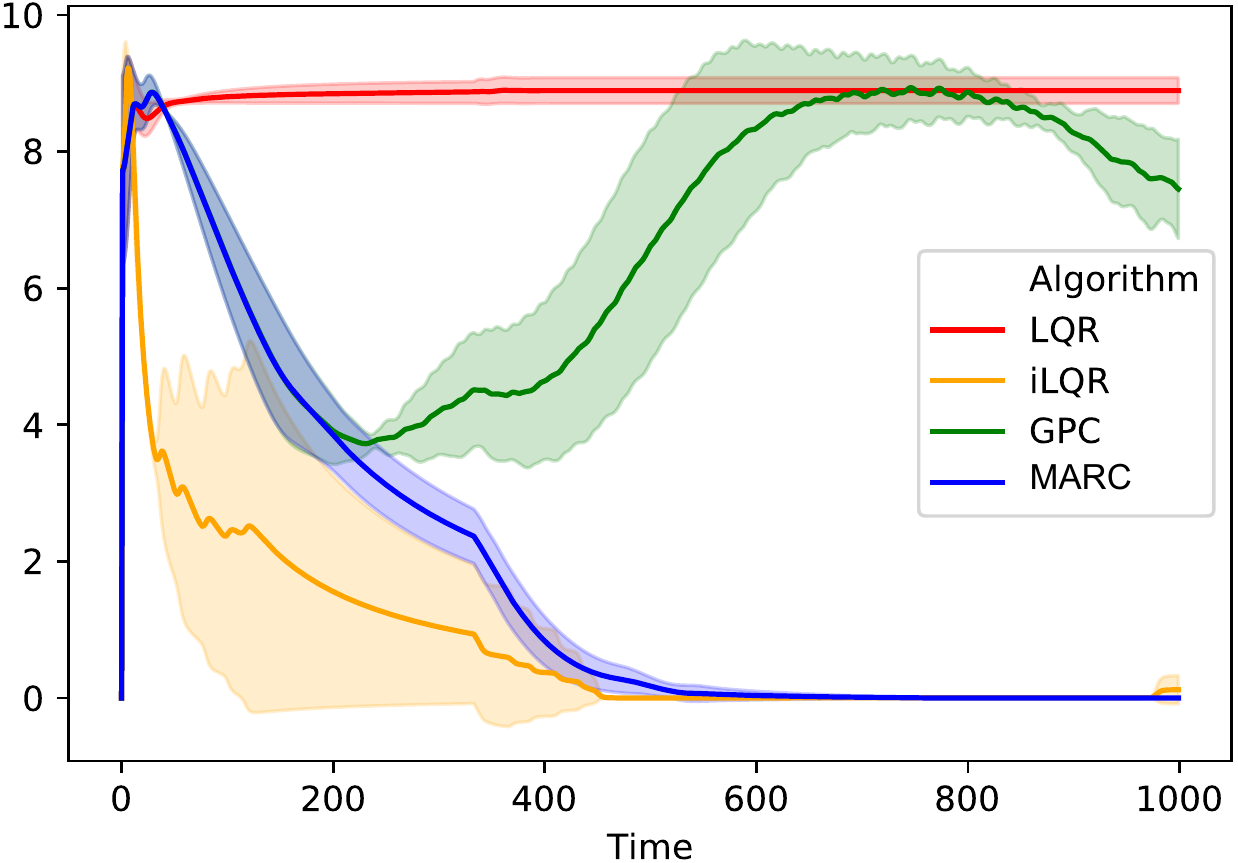}
\includegraphics[width=69mm]{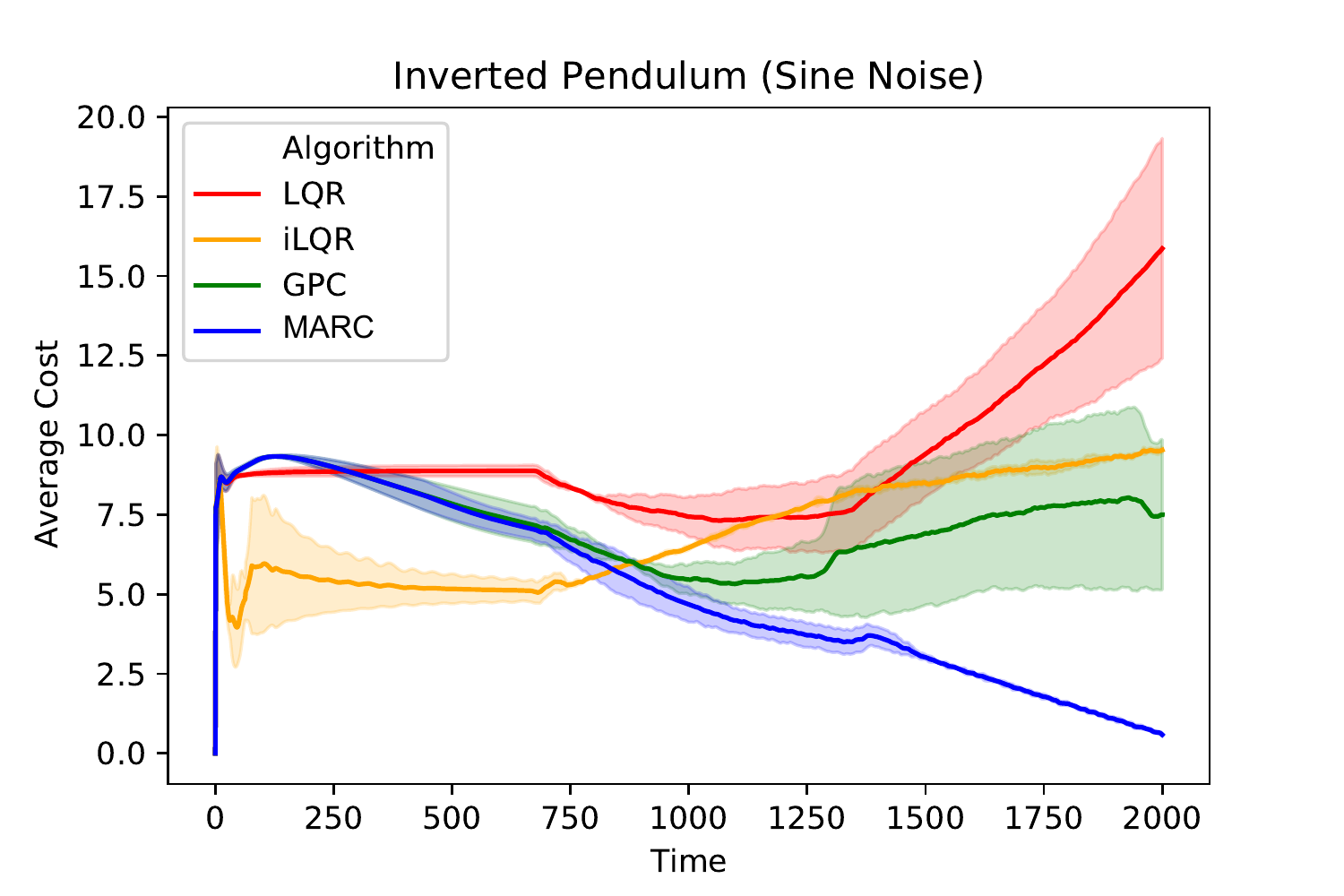}
\caption{Comparison of $T/3$-window averaged costs on a noiseless pendulum environment (left) and a pendulum environment experiencing a midway sinusoidal shock (right).} \label{pendulum}
\end{figure}

For MARC, we implement the efficient version of Algorithm \ref{alg:adaptive_control_main} using the GPC algorithm from \citet{agarwal2019logarithmic} as the base controller. As sanity checks, we compare our performance to GPC and to the linear controller LQR which acts according to the algebraic Ricatti equation computed at the start of the experiment. More relevantly, we also compare against iLQR, a planning method for non-linear control via iterative linearization, with the improvements described in \citet{iLQR}.

In the left plot of Figure \ref{pendulum}, we see that our method enables a controller originally developed for linear systems (GPC) to be used to solve this harder, nonlinear task, only slightly slower than the iLQR baseline. In the right plot, we see that iLQR is unable to adapt to an unanticipated shock due to its static and environment-agnostic design. Yet, our controller MARC demonstrates its robustness to the adversarial noise, and succeeds at this new harder task. More generally, we see that our algorithm works well in the setting of nonlinear control via iteratize linearization. These results confirm that the proposed approach is highly promising even from a practical standpoint, and provides a viable alternative to the classic planning approach.

We remark that there are numerous other planning methods that are robust to noise (see extensive survey \cite{MBBL}), which we did not evaluate: our experiments are geared to demonstrate the advantage of {\it  online control}. These methods do not need to know the dynamics ahead of time, and are less computation intensive than planning methods, as they consist of iterative gradient-based updates.

\section{Conclusion}\label{sec:conclusion}

We considered the control of time-varying linear dynamical systems from the perspective of online learning. Using tools from the theory of adaptive regret, we devise new efficient algorithms with provable guarantees in both online control and online prediction: they attain near-optimal \emph{first-order} regret bounds on any interval in time. 


In terms of future directions and open problems, it is interesting to extend our results to \emph{strongly} adaptive regret: in particular, it is interesting to answer the question whether strongly adaptive first-order regret, i.e. depending on the optimal cost per interval, can be achieved. This cannot be done trivially by the approach of \cite{daniely2015strongly} and answering this question would resolve optimality in this setting given our lower bound. The provided expected regret results can be stated with high probability using standard techniques with an additional $\sqrt{T}$ term in the regret. However, obtaining high probability bounds without impeding the first-order regret bound is quite more challenging and of independent interest to attain.

Finally, our guarantees hold w.r.t. adaptive, rather than oblivious adversaries, which is crucial for nonlinear control. It is interesting to map out which properties of nonlinear dynamics allow effective control via the  LTV approximation.

\newpage

\bibliographystyle{abbrvnat}
\bibliography{ref}

\begin{thebibliography}{44}
\providecommand{\natexlab}[1]{#1}
\providecommand{\url}[1]{\texttt{#1}}
\expandafter\ifx\csname urlstyle\endcsname\relax
  \providecommand{\doi}[1]{doi: #1}\else
  \providecommand{\doi}{doi: \begingroup \urlstyle{rm}\Url}\fi

\bibitem[Abbasi-Yadkori and Szepesv{\'a}ri(2011)]{abbasi2011regret}
Y.~Abbasi-Yadkori and C.~Szepesv{\'a}ri.
\newblock Regret bounds for the adaptive control of linear quadratic systems.
\newblock In \emph{Proceedings of the 24th Annual Conference on Learning
  Theory}, pages 1--26, 2011.

\bibitem[Adamskiy et~al.(2016)Adamskiy, Koolen, Chernov, and
  Vovk]{adamskiy2016closer}
D.~Adamskiy, W.~M. Koolen, A.~Chernov, and V.~Vovk.
\newblock A closer look at adaptive regret.
\newblock \emph{The Journal of Machine Learning Research}, 17\penalty0
  (1):\penalty0 706--726, 2016.

\bibitem[Agarwal et~al.(2019{\natexlab{a}})Agarwal, Bullins, Hazan, Kakade, and
  Singh]{agarwal2019online}
N.~Agarwal, B.~Bullins, E.~Hazan, S.~Kakade, and K.~Singh.
\newblock Online control with adversarial disturbances.
\newblock In \emph{International Conference on Machine Learning}, pages
  111--119, 2019{\natexlab{a}}.

\bibitem[Agarwal et~al.(2019{\natexlab{b}})Agarwal, Hazan, and
  Singh]{agarwal2019logarithmic}
N.~Agarwal, E.~Hazan, and K.~Singh.
\newblock Logarithmic regret for online control.
\newblock In \emph{Advances in Neural Information Processing Systems}, pages
  10175--10184, 2019{\natexlab{b}}.

\bibitem[Agarwal et~al.(2020)Agarwal, Brukhim, Hazan, and
  Lu]{agarwal2020boosting}
N.~Agarwal, N.~Brukhim, E.~Hazan, and Z.~Lu.
\newblock Boosting for control of dynamical systems, 2020.

\bibitem[Anava et~al.(2015)Anava, Hazan, and Mannor]{anava2015online}
O.~Anava, E.~Hazan, and S.~Mannor.
\newblock Online learning for adversaries with memory: price of past mistakes.
\newblock In \emph{Advances in Neural Information Processing Systems}, pages
  784--792, 2015.

\bibitem[Bansal et~al.(2017)Bansal, Chen, Herbert, and
  Tomlin]{bansal2017hamilton}
S.~Bansal, M.~Chen, S.~Herbert, and C.~J. Tomlin.
\newblock Hamilton-jacobi reachability: A brief overview and recent advances.
\newblock In \emph{2017 IEEE 56th Annual Conference on Decision and Control
  (CDC)}, pages 2242--2253. IEEE, 2017.

\bibitem[Blondel and Tsitsiklis(2000)]{blondel2000survey}
V.~D. Blondel and J.~N. Tsitsiklis.
\newblock A survey of computational complexity results in systems and control.
\newblock \emph{Automatica}, 36\penalty0 (9):\penalty0 1249--1274, 2000.

\bibitem[Bousquet and Warmuth(2002)]{bousquet2002tracking}
O.~Bousquet and M.~K. Warmuth.
\newblock Tracking a small set of experts by mixing past posteriors.
\newblock \emph{Journal of Machine Learning Research}, 3\penalty0
  (Nov):\penalty0 363--396, 2002.

\bibitem[Bradbury et~al.(2018)Bradbury, Frostig, Hawkins, Johnson, Leary,
  Maclaurin, Necula, Paszke, Vander{P}las, Wanderman-{M}ilne, and
  Zhang]{jax2018github}
J.~Bradbury, R.~Frostig, P.~Hawkins, M.~J. Johnson, C.~Leary, D.~Maclaurin,
  G.~Necula, A.~Paszke, J.~Vander{P}las, S.~Wanderman-{M}ilne, and Q.~Zhang.
\newblock {JAX}: composable transformations of {P}ython+{N}um{P}y programs,
  2018.
\newblock URL \url{http://github.com/google/jax}.

\bibitem[Brockman et~al.(2016)Brockman, Cheung, Pettersson, Schneider,
  Schulman, Tang, and Zaremba]{brockman2016openai}
G.~Brockman, V.~Cheung, L.~Pettersson, J.~Schneider, J.~Schulman, J.~Tang, and
  W.~Zaremba.
\newblock Openai gym, 2016.

\bibitem[Budisic et~al.(2012)Budisic, Mohr, and Mezic]{Koopmanism}
M.~Budisic, R.~Mohr, and I.~Mezic.
\newblock Applied koopmanism.
\newblock \emph{Chaos: An Interdisciplinary Journal of Nonlinear Science},
  22\penalty0 (4), Dec 2012.

\bibitem[Cesa-Bianchi and Lugosi(2006)]{cesa2006prediction}
N.~Cesa-Bianchi and G.~Lugosi.
\newblock \emph{Prediction, learning, and games}.
\newblock Cambridge university press, 2006.

\bibitem[Cohen et~al.(2018)Cohen, Hasidim, Koren, Lazic, Mansour, and
  Talwar]{cohen2018online}
A.~Cohen, A.~Hasidim, T.~Koren, N.~Lazic, Y.~Mansour, and K.~Talwar.
\newblock Online linear quadratic control.
\newblock In \emph{International Conference on Machine Learning}, pages
  1029--1038, 2018.

\bibitem[Cohen et~al.(2019)Cohen, Koren, and Mansour]{cohen2019learning}
A.~Cohen, T.~Koren, and Y.~Mansour.
\newblock Learning linear-quadratic regulators efficiently with only $\sqrt{T}$
  regret.
\newblock In \emph{International Conference on Machine Learning}, pages
  1300--1309, 2019.

\bibitem[Daniely et~al.(2015)Daniely, Gonen, and
  Shalev-Shwartz]{daniely2015strongly}
A.~Daniely, A.~Gonen, and S.~Shalev-Shwartz.
\newblock Strongly adaptive online learning.
\newblock In \emph{International Conference on Machine Learning}, pages
  1405--1411. PMLR, 2015.

\bibitem[Dean et~al.(2018)Dean, Mania, Matni, Recht, and Tu]{dean2018regret}
S.~Dean, H.~Mania, N.~Matni, B.~Recht, and S.~Tu.
\newblock Regret bounds for robust adaptive control of the linear quadratic
  regulator.
\newblock In \emph{Advances in Neural Information Processing Systems}, pages
  4188--4197, 2018.

\bibitem[Geulen et~al.(2010)Geulen, V{\"o}cking, and Winkler]{geulen2010regret}
S.~Geulen, B.~V{\"o}cking, and M.~Winkler.
\newblock Regret minimization for online buffering problems using the weighted
  majority algorithm.
\newblock In \emph{COLT}, pages 132--143, 2010.

\bibitem[Goel and Hassibi(2021)]{goel2021regretoptimal}
G.~Goel and B.~Hassibi.
\newblock Regret-optimal control in dynamic environments, 2021.

\bibitem[Hazan(2016)]{hazan2016introduction}
E.~Hazan.
\newblock Introduction to online convex optimization.
\newblock \emph{Foundations and Trends{\textregistered} in Optimization},
  2\penalty0 (3-4):\penalty0 157--325, 2016.

\bibitem[Hazan(2020)]{hazan2020lecture}
E.~Hazan.
\newblock Lecture notes: Computational control theory.
\newblock \url{https://sites.google.com/view/cos59x-cct/lecture-notes}, 2020.
\newblock [Online; accessed 15-Jan-2021].

\bibitem[Hazan and Seshadhri(2009)]{hazan2009efficient}
E.~Hazan and C.~Seshadhri.
\newblock Efficient learning algorithms for changing environments.
\newblock In \emph{Proceedings of the 26th annual international conference on
  machine learning}, pages 393--400. ACM, 2009.

\bibitem[Hazan and Singh(2021)]{hazan2021tutorial}
E.~Hazan and K.~Singh.
\newblock Tutorial: online and non-stochastic control, July 2021.

\bibitem[Hazan et~al.(2019)Hazan, Kakade, and Singh]{hazan2019nonstochastic}
E.~Hazan, S.~M. Kakade, and K.~Singh.
\newblock The nonstochastic control problem, 2019.

\bibitem[Herbster and Warmuth(1998)]{herbster1998tracking}
M.~Herbster and M.~K. Warmuth.
\newblock Tracking the best expert.
\newblock \emph{Machine learning}, 32\penalty0 (2):\penalty0 151--178, 1998.

\bibitem[Lale et~al.(2020{\natexlab{a}})Lale, Azizzadenesheli, Hassibi, and
  Anandkumar]{anima1}
S.~Lale, K.~Azizzadenesheli, B.~Hassibi, and A.~Anandkumar.
\newblock Regret bound of adaptive control in linear quadratic gaussian (lqg)
  systems, 2020{\natexlab{a}}.

\bibitem[Lale et~al.(2020{\natexlab{b}})Lale, Azizzadenesheli, Hassibi, and
  Anandkumar]{anima2}
S.~Lale, K.~Azizzadenesheli, B.~Hassibi, and A.~Anandkumar.
\newblock Logarithmic regret bound in partially observable linear dynamical
  systems, 2020{\natexlab{b}}.

\bibitem[Lale et~al.(2020{\natexlab{c}})Lale, Azizzadenesheli, Hassibi, and
  Anandkumar]{anima3}
S.~Lale, K.~Azizzadenesheli, B.~Hassibi, and A.~Anandkumar.
\newblock Regret minimization in partially observable linear quadratic control,
  2020{\natexlab{c}}.

\bibitem[Li et~al.(2019)Li, Chen, and Li]{li2019online}
Y.~Li, X.~Chen, and N.~Li.
\newblock Online optimal control with linear dynamics and predictions:
  Algorithms and regret analysis.
\newblock \emph{Advances in Neural Information Processing Systems},
  32:\penalty0 14887--14899, 2019.

\bibitem[Majumdar et~al.(2020)Majumdar, Hall, and Ahmadi]{majumdar2020recent}
A.~Majumdar, G.~Hall, and A.~A. Ahmadi.
\newblock Recent scalability improvements for semidefinite programming with
  applications in machine learning, control, and robotics.
\newblock \emph{Annual Review of Control, Robotics, and Autonomous Systems},
  3:\penalty0 331--360, 2020.

\bibitem[Mania et~al.(2019)Mania, Tu, and Recht]{mania2019certainty}
H.~Mania, S.~Tu, and B.~Recht.
\newblock Certainty equivalence is efficient for linear quadratic control.
\newblock In \emph{Advances in Neural Information Processing Systems}, pages
  10154--10164, 2019.

\bibitem[Moore(2012)]{moore2012iterative}
K.~L. Moore.
\newblock \emph{Iterative learning control for deterministic systems}.
\newblock Springer Science \& Business Media, 2012.

\bibitem[Rowley and Dawson(2017)]{clancy}
C.~W. Rowley and S.~T. Dawson.
\newblock Model reduction for flow analysis and control.
\newblock \emph{Annual Review of Fluid Mechanics}, 49\penalty0 (1):\penalty0
  387--417, 2017.

\bibitem[Simchowitz(2020)]{simchowitz2020making}
M.~Simchowitz.
\newblock Making non-stochastic control (almost) as easy as stochastic, 2020.

\bibitem[Simchowitz et~al.(2020)Simchowitz, Singh, and
  Hazan]{simchowitz2020improper}
M.~Simchowitz, K.~Singh, and E.~Hazan.
\newblock Improper learning for non-stochastic control, 2020.

\bibitem[Stengel(1994)]{Stengel1994OptimalCA}
R.~F. Stengel.
\newblock \emph{Optimal control and estimation}.
\newblock Courier Corporation, 1994.

\bibitem[{Tassa} et~al.(2012){Tassa}, {Erez}, and {Todorov}]{iLQR}
Y.~{Tassa}, T.~{Erez}, and E.~{Todorov}.
\newblock Synthesis and stabilization of complex behaviors through online
  trajectory optimization.
\newblock In \emph{2012 IEEE/RSJ International Conference on Intelligent Robots
  and Systems}, pages 4906--4913, 2012.

\bibitem[Todorov and Li(2005)]{todorov2005generalized}
E.~Todorov and W.~Li.
\newblock A generalized iterative lqg method for locally-optimal feedback
  control of constrained nonlinear stochastic systems.
\newblock In \emph{Proceedings of the 2005, American Control Conference,
  2005.}, pages 300--306. IEEE, 2005.

\bibitem[Wang et~al.(2019)Wang, Bao, Clavera, Hoang, Wen, Langlois, Zhang,
  Zhang, Abbeel, and Ba]{MBBL}
T.~Wang, X.~Bao, I.~Clavera, J.~Hoang, Y.~Wen, E.~Langlois, S.~Zhang, G.~Zhang,
  P.~Abbeel, and J.~Ba.
\newblock Benchmarking model-based reinforcement learning, 2019.

\bibitem[Zhang et~al.(2018)Zhang, Yang, Zhou, et~al.]{zhang2018dynamic}
L.~Zhang, T.~Yang, Z.-H. Zhou, et~al.
\newblock Dynamic regret of strongly adaptive methods.
\newblock In \emph{International conference on machine learning}, pages
  5882--5891. PMLR, 2018.

\bibitem[Zhang et~al.(2019)Zhang, Liu, and Zhou]{zhang2019adaptive}
L.~Zhang, T.-Y. Liu, and Z.-H. Zhou.
\newblock Adaptive regret of convex and smooth functions.
\newblock \emph{arXiv preprint arXiv:1904.11681}, 2019.

\bibitem[Zhang et~al.(2021)Zhang, Cutkosky, and Paschalidis]{zhang2021strongly}
Z.~Zhang, A.~Cutkosky, and I.~C. Paschalidis.
\newblock Strongly adaptive oco with memory.
\newblock \emph{arXiv preprint arXiv:2102.01623}, 2021.

\bibitem[Zhou et~al.(1996)Zhou, Doyle, and Glover]{kemin}
K.~Zhou, J.~C. Doyle, and K.~Glover.
\newblock \emph{Robust and Optimal Control}.
\newblock Prentice-Hall, Inc., USA, 1996.
\newblock ISBN 0134565673.

\bibitem[Zinkevich(2003)]{zinkevich2003online}
M.~Zinkevich.
\newblock Online convex programming and generalized infinitesimal gradient
  ascent.
\newblock In \emph{Proceedings of the Twentieth International Conference on
  International Conference on Machine Learning}, ICML'03, page 928–935. AAAI
  Press, 2003.
\newblock ISBN 1577351894.

\end{thebibliography}

\newpage

\appendix

\tableofcontents

\newpage

\section*{Notation}
Throughout this work, we use $[n] = [1, n]$ as a shorthand. The norm $\| \cdot \|$ refers to Euclidean/spectral norms unless stated otherwise. We denote $\rho(\cdot)$ to be the spectral radius of a matrix.

The notation $D$ refers to the diameter of the convex constraint set $\mathcal{K}$ and $G$ refers to the gradient norm bound of functions $\tilde{f}_t$ in the OCO (with memory) setting.

For an online algorithm $\mathcal{A}$, by abuse of notation and for ease of algebraic manipulations, denote $\mathcal{R}_{\mathcal{A}}(T)$ to be the regret, $\mathcal{S}_{\mathcal{A}}(T)$ to be the action shift, and $\text{Ad}\mathcal{R}_{\mathcal{A}}(T)$ to be the adaptive regret of the algorithm over a time interval of length $T$. The same notation applies to a controller $\mathcal{C}$.

The Big-Oh notation $O(\cdot)$ hides problem and absolute constants, with $o(\cdot), \Omega(\cdot), \Theta(\cdot)$ following suit, $\tilde{O}(\cdot)$ hides terms poly-logarithmic in $T$, $\mathcal{O}(\cdot)$ hides all irrelevant problem parameters (anything not depending on $T$), and $\text{poly}(\cdot)$ denotes a function with polynomial complexity over its argument.

\section{Adaptive Regret Results for OCO with Memory}\label{sec:ocom_adareg}

\subsection{Adaptive regret for functions with memory} \label{sec:ada_regret}
In this section, we provide a new meta-algorithm for OCO with memory with adaptive regret guarantees over bounded loss functions. The main benefit of the algorithm is its low action shift due to the shrinking technique \citep{geulen2010regret}. This allows for novel adaptive regret results over (strongly convex) functions with memory. The meta-algorithm, essentially an expert switching scheme, is presented in detail in Algorithm \ref{alg:adaptive_reg}.

\begin{algorithm}
\caption{Meta Adaptive Regret Algorithm (MARA)}\label{alg:adaptive_reg}
\begin{algorithmic}[1]
\STATE \textbf{Input:} action set $\mathcal{K}$, number of rounds $T$, $N$ online algorithms $\mathcal{A}_1,...,\mA_N$, parameters $\eta, \sigma$
\STATE \textbf{Setup:} pick arbitrary $z_0 \in \mathcal{K}$, assign $w_1^i = 1$ and denote feedback $\mathcal{F}_0^i = \{z_0\}$, $\forall i \in [N]$

\FOR{$t = 1, ..., T$}
\STATE compute $z_t^i = \mathcal{A}_i(\mathcal{F}_{t-1}^i)$ for all $i \in [N]$, denote $W_t = \sum_{i=1}^N w_t^i$
\IF{$t = 1$}
    \STATE choose $i_t = i$ w.p. $p_t^i = w_t^i / W_t$ for all $i \in [N]$
\ELSE
    \STATE keep $i_t = i_{t-1}$ w.p. $w_t^{i_{t-1}} / w_{t-1}^{i_{t-1}}$, o.w. choose $i_t = i$ w.p. $p_t^i = w_t^i / W_t$ for all $i \in [N]$
\ENDIF
\STATE play $z_t = z_t^{i_t}$, suffer loss $f_t(z_{t-H:t})$, observe $\tilde{f}_t(\cdot)$ and compute $\tilde{f}_t(z_t^i)$ for all $i \in [N]$
\STATE compute $\overline{w}_{t+1}^i = w_t^i e^{-\eta \tilde{f}_t(z_t^i)}$ and $w_{t+1}^i = (1-\sigma) \overline{w}_{t+1}^i + \sigma \overline{W}_{t+1} / N$ for all $i \in [N]$ with $\overline{W}_{t+1} = \sum_{i=1}^N \overline{w}_{t+1}^i$
\STATE update $\mathcal{F}_t^i = \mathcal{F}_{t-1}^i \cup \{ z_t^i, \tilde{f}_t \}$ for all $i \in [N]$
\ENDFOR
\end{algorithmic}
\end{algorithm}

\begin{theorem}\label{thm:meta}
Let $f_t$ be coordinate-wise $L$-Lipschitz with $H$ memory and range $[0,1]$. For $\sigma=1/T$, Algorithm \ref{alg:adaptive_reg} achieves the following bound for any experts $\{\mA_i\}$ over any interval $I = [r, s] \subseteq [T]$:
\begin{equation*}
    \sum_{t=r}^s \E[f_t(z_{t-H:t})] \leq \gamma_{\eta} \min_{i \in [N]} \sum_{t=r}^s \tilde{f}_t(z_t^i) + \left( \frac{4}{\eta} + 5 H^2 L \mathcal{S}_{\mA}(|I|) \right) \log (T N) 
\end{equation*}
where we denote $\mathcal{S}_{\mA}(T) = \max\limits_{i \in [N]} \mathcal{S}_{\mA_i}(T)$ and $\gamma_\eta = 1 + 4 H^2 L \mathcal{S}_{\mA}(|I|) \eta$.
\end{theorem}

The most interesting application of this theorem is for strongly convex functions, as given below. 

\begin{theorem}\label{thm:adaoco}
Let $\tilde{f}_t$ be strongly convex with bounded gradient norms. Take $N = T$ with $\mA_i$ starting $\mA_{scOGD}$ (see Fact \ref{claim:scogd}) with initial point $z_0$ at time $t=i$ while outputting constant $z_0$ at times $t < i$ for each $i \in [N]$. Assign the parameter value $\eta = (4 H^2 L \mathcal{S}_{\mA}(T) \log T)^{-1}$ to get
\begin{equation}\label{eq:logfactor}
    \forall I = [r, s] \subseteq [T], \quad \sum_{t=r}^s \E[f_t(z_{t-H:t})] \leq \left( 1 + \frac{1}{\log T} \right) \min_{z \in \mathcal{K}} \sum_{t=r}^s \tilde{f}_t(z) + O(H^2 L \log^3 T) ~.
\end{equation}
Furthermore, if $\text{OPT} = \min\limits_{z \in \mathcal{K}} \sum_{t=1}^T \tilde{f}_t(z)$ is known, then with $\eta = (4 H^2 L \mathcal{S}_{\mathcal{A}}(T) \sqrt{OPT})^{-1}$, we get 
\begin{equation}\label{eq:optregret}
   \forall I = [r, s] \subseteq [T], \quad \sum_{t=r}^s \E[f_t(z_{t-H:t})] \leq \min_{z \in \mathcal{K}} \sum_{t=r}^s \tilde{f}_t(z) + \tilde{\O}(H^2 L \cdot \sqrt{OPT}) ~.
\end{equation}
\end{theorem}

\begin{remark}\label{remark:eta}
The first-order bound in \eqref{eq:optregret} can be achieved by adaptively updating $\eta$ with no prior knowledge of $OPT$ and no additional complexity or regret overhead asymptotically (see section \ref{sec:param_choice}).
\end{remark}

\subsubsection{Shrinking argument and adaptive regret}
The shrinking technique in Algorithm \ref{alg:adaptive_reg} ensures a limited number of switches between experts; MARA keeps the previous expert with a certain probability. This is necessary to achieve low action shift and extend the adaptive regret result to functions with memory. However, this modification might potentially hurt the performance with respect to the best expert. The next lemma shows that the shrinking technique doesn't actually alter the algorithm in terms of performance.

\begin{lemma}\label{lem:shrinking}
The expert probabilities stay invariant $\Pr[i_t = i] = p_t^i = w_t^i/W_t$, for all $i \in [N], t \in [T]$.
\end{lemma}
\begin{proof}
Our proof is via an inductive argument over $t = 1, \dots, T$. By construction, for $t=1$ we have that $\Pr[i_t = i] = p_t^i$. Assume this holds for rounds up to $t-1$ and let us show the same for $t$. There are two cases for $i_t = i$: (i) either $i_{t-1} = i$ and there was no expert switch; (ii) or there was an expert switch from $i_{t-1} = j$ to $i_t = i$ via sampling by $p_t$.
\begin{align*}
    \Pr[i_t = i] &= \Pr[i_{t-1} = i] \cdot \frac{w_t^i}{w_{t-1}^i} + \sum_{j=1}^N \Pr[i_{t-1} = j] \cdot \left(1 - \frac{w_t^j}{w_{t-1}^j} \right) \cdot \frac{w_t^i}{W_t} = \\
    &= \frac{w_{t-1}^i}{W_{t-1}} \cdot \frac{w_t^i}{w_{t-1}^i} + \frac{w_t^i}{W_t} \cdot \sum_{j=1}^N \frac{w_{t-1}^j}{W_{t-1}} \cdot \left(1 - \frac{w_t^j}{w_{t-1}^j} \right) = \\
    &= \frac{w_t^i}{W_{t-1}} + \frac{w_t^i}{W_t} - \frac{w_t^i}{W_{t-1}} = \frac{w_t^i}{W_t} = p_t^i ~.
\end{align*}
\end{proof}

We now proceed to prove adaptive regret for the surrogate functions $\tilde{f}_t$. Since the expert probabilities are unchanged, according to the kept weights, the proof follows a standard methodology.

\begin{lemma}\label{lem:adaptive} Let $\tilde{f}_t$ have range $[0,1]$ and set $\sigma = 1/T$. Then for any interval $I = [r, s] \subseteq [T]$
\begin{equation}\label{eq:adaptive}
    (1-\eta) \cdot \sum_{t=r}^s \E[\tilde{f}_t(z_t)] \leq \min_{i \in [N]} \sum_{t=r}^s \tilde{f}_t(z_t^i) + \frac{1}{\eta}  \left[ \log (TN) + 2 \right] ~.
\end{equation}
\end{lemma}
\begin{proof}
Given that $z_t = z_t^{i_t}$ and $\Pr[i_t = i] = p_t^i$ then $\E[\tilde{f}_t(z_t)] = p_t^{\top} \tilde{f}_t$ where we denote $p_t, \tilde{f}_t$ to be $N$-dimensional vectors with entries $p_t^i, \tilde{f}_t(z_t^i)$ for $i \in [N]$. Notice that for any $t \in [T]$, by construction, $\overline{W}_{t} = W_t$ and $W_{t+1} / W_t = \sum_{i=1}^N p_t^i e^{- \eta \tilde{f}_t(z_t^i)}$. Using the inequalities $e^{-a} \leq 1 - a + a^2 \leq e^{-a + a^2}$ we obtain $W_{t+1}/W_t \leq \exp(- \eta p_t^{\top} \tilde{f}_t + \eta^2 p_t^{\top} \tilde{f}_t^2)$ which over the interval $I=[r, s]$ implies $W_{s+1}/W_r \leq \exp(-\eta \sum_{t=r}^s p_t^{\top} \tilde{f}_t + \eta^2 \sum_{t=r}^s p_t^{\top} \tilde{f}_t^2)$. On the other hand, for any $t \in [T]$ and any fixed $i \in [N]$ the weight construction implies $w_{t+1}^i \geq w_t^i (1-\sigma) e^{-\eta \tilde{f}_t(z_t^i)}$ which means $W_{s+1} / W_r \geq w_{s+1}^i / W_r \geq p_t^i (1-\sigma)^{|I|} \exp(- \eta \sum_{t=r}^s \tilde{f}_t(z_t^i))$. Combining and taking logarithm on both sides yields
\begin{equation*}
    \log p_t^i + |I| \log (1-\sigma) - \eta \sum_{t=r}^s \tilde{f}_t(z_t^i) \leq  \log(W_{s+1}/W_r) \leq - \eta \sum_{t=r}^s p_t^{\top} \tilde{f}_t + \eta^2 \sum_{t=r}^s p_t^{\top} \tilde{f}_t^2 ~.
\end{equation*}
By definition $p_t^i \geq \sigma/N$ and for $\sigma \leq 1/2$ one can use the bound $\log(1-\sigma) \geq -2 \sigma$. Hence, taking $\sigma = 1/T$, using $|I| \leq T$ and $\tilde{f}_t^2 \leq \tilde{f}_t$ element-wise along with simplifications of the inequality above results in the final bound
\begin{equation*}
    (1-\eta) \cdot \sum_{t=r}^s p_t^{\top} \tilde{f}_t \leq \sum_{t=r}^s \tilde{f}_t(z_t^i) + \frac{1}{\eta} \left[ \log (T N) + 2 \right] ~.
\end{equation*}
Since this holds for any fixed $i \in [N]$, we take $i^* = \argmin_{i \in [N]} \sum_{t=r}^s \tilde{f}_t(z_t^i)$ and conclude the lemma statement.
\end{proof}

\subsubsection{Action Shift}
The essential property of Algorithm \ref{alg:adaptive_reg} that yields adaptive regret guarantees for functions with memory is low action shift. This is shown by first bounding the number of expert switches of the algorithm, a byproduct of the shrinking technique.
\begin{lemma}\label{lem:switches}
Fix any $I = [r, s]$ interval, and denote $k_I = \sum_{t \in I} \mathbf{1}_{i_{t+1} \neq i_t} $ to be the number of switches in the interval. The expected number of such switches is bounded as follows
\begin{equation}\label{eq:switches}
    \E[k_I ] \leq \eta \min_{i \in [N]} \sum_{t=r}^s \tilde{f}_t(z_t^i) + \log (T N) + 2 ~.
\end{equation}
\end{lemma}
\begin{proof}
Denote $q_{t+1} = \Pr[i_{t+1} \neq i_t]$ which can be obtained by counting switching probability for each case $i_t = i$ for all $i \in [N]$. Switching probability is bounded by $1 - w_{t+1}^i/w_t^i$ for each $i$, hence 
$$q_{t+1} \leq \sum_{i=1}^N \Pr[i_t = i] \cdot \left(1 - \frac{w_{t+1}^i}{w_t^i} \right) = \sum_{i=1}^T \left( \frac{w_t^i}{W_t} - \frac{w_t^i}{W_t} \cdot \frac{w_{t+1}^i}{w_t^i} \right) = 1 - \frac{W_{t+1}}{W_t} ~.$$
Using the $1-a \leq e^{-a}$ inequality we obtain $W_{s+1}/W_r \leq \prod_{t=r}^s (1-q_{t+1}) \leq \exp(-\sum_{t=r}^s q_{t+1})$, so we conclude that
\begin{equation*}
    \E[k_I] = \sum_{t=r}^s q_{t+1} \leq \log(W_r/W_{s+1}) \leq \eta \min_{i \in [N]} \sum_{t=r}^s \tilde{f}_t(z_t^i) + \log (T N) + 2,
\end{equation*}
where the last bound was shown in Lemma \ref{lem:adaptive}.
\end{proof}
Given the result of Lemma \ref{lem:switches}, let us derive interval specific action shift bounds for Algorithm \ref{alg:adaptive_reg} based on the maximum action shift of the experts $\{\mA_i\}_{i \in [N]}$ defined as $\mathcal{S}_{\mA}(T) = \max_{i \in [N]} \mathcal{S}_{\mA_i}(T)$.
\begin{lemma}\label{lem:stability}
Fix any $I = [r, s] \subseteq [T]$, assume $\mathcal{S}_{\mathcal{A}}(|I|) \geq D$ for simplicity. Then, 
\begin{equation}\label{eq:stability}
    \sum_{t=r}^s \E\left[ \| z_{t+1} - z_t \| \right] \leq 2 \eta \mathcal{S}_{\mathcal{A}}(|I|) \cdot \min_{i \in [N]} \sum_{t=r}^s \tilde{f}_t(z_t^i) + 2 \mathcal{S}_{\mathcal{A}}(|I|) \log (N T) + 3 \mathcal{S}_{\mA}(|I|) ~.
\end{equation}
\end{lemma}
\begin{proof}
The statement above is proved by showing $\sum_{t=r}^s \E\left[ \| z_{t+1} - z_t \| \right] \leq (2\E[k_I] + 1)\shift_{\A}(|I|)$ and using the result from Lemma \ref{lem:switches}. Denote $k = k_I$ and $j_1, \dots, j_k$ to be the indices for $k$ switches in $I$, i.e. $i_{j_c+1} \neq i_{j_c}$ for all $c = 1, \dots, k$. Let $j_0 = r-1, j_{k+1} = s+1$. Divide $|I|$ into $k+1$ uninterrupted intervals $[j_{c-1}+1, j_c]$ of lengths $l_{c-1}$ for $c \in [k+1]$ such that $\sum_{c=0}^{k} l_c = |I|-k$. The rest is $k$ $1$-length intervals $[j_c, j_c+1]$ accounting for the switches between experts the shift for which is bounded by the diameter $D$. On the other hand, the shift in each of the uninterrupted intervals is bounded by the action shift $\mathcal{S}_{\mathcal{A}}(\cdot)$ of the corresponding length, which is, by definition, non-decreasing with respect to the interval length. We bound the overall action shift using monotonicity as follows
\begin{equation*}
    \sum_{t=r}^s \| z_{t+1} - z_t \| \leq \sum_{c=0}^{k} \mathcal{S}_{\mathcal{A}}(l_c) + k D \leq (k+1)  \mathcal{S}_{\mathcal{A}} (|I|) + k D \leq (2k+1) \shift_{\mA}(|I|) ~.
\end{equation*}
Since the right-hand side is linear in $k$ we can simply take expectation over both sides and use the bound on $\E[k_I]$ in \eqref{eq:switches} to conclude the lemma statement.
\end{proof}

\subsubsection{Proof of Theorems \ref{thm:meta}, \ref{thm:adaoco}}
To conclude the theorem results, let us connect the notions of adaptive regret and action shift together in order to relate them to adaptive regret for functions with memory. The following lemma does just that and the rest of the proof of Theorem \ref{thm:meta}, and consequently Theorem \ref{thm:adaoco}, is technicalities.
\begin{lemma}\label{lem:memadregret}
Let $f_t$ be coordinate-wise $L$-Lipschitz functions with $H$ memory. Suppose an online algorithm $\mA$ over surrogate losses $\tilde{f}_t$ has adaptive regret $\adreg_{\mA}(\cdot)$ and action shift $\shift_{\mA}(\cdot)$. Then, (expected) adaptive regret over $f_t$ is bounded as follows: over any $I = [r, s] \subseteq [T]$ interval,
\begin{equation}\label{eq:memadregret1}
    \sum_{t=r}^s \left| f_t(z_{t-H:t}) - \tilde{f}_t(z_t) \right| \leq H^2 L \sum_{t=r-H}^s \| z_{t+1} - z_t \| \leq H^3 L D + H^2 L \sum_{t=r}^s \| z_{t+1} - z_t \| ~.
\end{equation}
\begin{equation}\label{eq:memadregret2}
    \sum_{t=r}^s f_t(z_{t-H:t}) - \min_{z \in \mathcal{K}} \sum_{t=r}^s \tilde{f}_t(z) \leq \adreg_{\mA}(T) + H^2 L \cdot \shift_{\mA}(T) ~.
\end{equation}
\end{lemma}
\begin{proof}
Use the coordinate-wise Lipschitz property of the loss functions $f_t$ to relate them to the surrogate loss $\tilde{f}_t$. For any $t \in [T]$, using triangle inequality we get
\begin{equation*}
    \left| f_t(z_{t-H:t}) - \tilde{f}_t(z_t) \right| \leq L \sum_{h=1}^H \|z_t - z_{t-h}\| \leq HL \sum_{h=1}^H \| z_{t-h+1} - z_{t-h} \| ~.
\end{equation*}
Summing up the iterations $t=r$ to $t=s$ we obtain the bound given in \eqref{eq:memadregret1}. By definition, we have that $\sum_{t \in I} \tilde{f}_t(z_t) \leq \min_{z \in \mathcal{K}} \sum_{t \in I} \tilde{f}_t(z) + \adreg_{\mA}(T)$. Furthermore, $\shift_{\mA}(T)$ is an upper bound on overall action shift, hence we obtain \eqref{eq:memadregret2} from \eqref{eq:memadregret1}. Note that \eqref{eq:memadregret2} and \eqref{eq:memadregret1} still hold in expectation if $\mA$ is randomized.
\end{proof}
\begin{proof}[Proof of Theorem \ref{thm:meta}]
The main statement of the theorem regarding $\mA_{\text{MARA}}$ is attained by putting together the results \eqref{eq:adaptive}, \eqref{eq:stability}, \eqref{eq:memadregret1} from Lemmas \ref{lem:adaptive}, \ref{lem:stability}, \ref{lem:memadregret} using the following simplifying assumptions: for $\eta \leq 1/2$, we have $1/(1-\eta) \leq 1 + 2 \eta \leq 2$; $H^2 L \shift_{\A}(|I|) \geq 1$; and $\max(H, 2) \leq \log (TN)$. These assumptions are made to ease the exposition clarity and do not affect the provided results asymptotically.
\end{proof}

\begin{proof}[Proof of Theorem \ref{thm:adaoco}]
Use the result from Theorem \ref{thm:meta} with $N = T$ and for each $i \in [N]$ take $\mA_i$ to be $\mA_{\text{scOGD}}$ starting from round $t=i$ (it outputs initial $z_0$ up to that round). Each action shift is bounded by $\shift_{\mA_i}(T) \leq \shift_{\mA_{\text{scOGD}}}(T)$ for all $i \in [N]$, yielding $\shift_{\mA}(T) = \shift_{\mA_{\text{scOGD}}}(T) = O(\log T)$ according to Fact \ref{claim:scogd}. Furthermore, the regret guarantee from Fact \ref{claim:scogd} implies 
\begin{equation*}
\min_{i \in [N]} \sum_{t=r}^s \tilde{f}_t(z_t^i) \leq \sum_{t=r}^s \tilde{f}_t(z_t^r) \leq \min_{z \in \mathcal{K}} \sum_{t=r}^s \tilde{f}_t(z) + \reg_{\mA_{\text{scOGD}}}(|I|) \leq \min_{z \in \mathcal{K}} \sum_{t=r}^s \tilde{f}_t(z) + O(\log T) ~.
\end{equation*}
Assuming $\log T \geq 1$, the specified value for $\eta = (4 H^2 L \mathcal{S}_{\mathcal{A}}(|I|) \log (T))^{-1} \leq 1/4$ implies that $\gamma_{\eta} = 1 + 1/ \log T = O(1)$. The specific regret, action shift and parameter values conclude \eqref{eq:logfactor}.

Given the value of $OPT = \min\limits_{z \in \mathcal{K}} \sum\limits_{t=1}^T \tilde{f}_t(z)$, denote $z^* = \argmin\limits_{z \in \mathcal{K}} \sum\limits_{t=1}^T \tilde{f}_t(z)$, and notice that for any $I = [r, s] \subseteq [T]$ the best in hindsight loss of that interval is smaller than that of $[T]$, 
\begin{equation*}
    \min_{z \in \mathcal{K}} \sum_{t=r}^s \tilde{f}_t(z) \leq \sum_{t=r}^s \tilde{f}_t(z^*) \leq \min_{z \in \mathcal{K}} \sum_{t=1}^T \tilde{f}_t(z) = OPT ~.
\end{equation*}
The bound above is due to the nonnegativity assumption for $\tilde{f}_t$. Assuming $OPT \geq 1$, the specified value for $\eta = (4 H^2 L \mathcal{S}_{\mathcal{A}}(|I|) \sqrt{OPT})^{-1} \leq 1/4$ implies that $\gamma_{\eta} = 1 + 1/ \sqrt{OPT} = O(1)$. Analogously due to Fact \ref{claim:scogd}, the specific regret, action shift and parameter values conclude \eqref{eq:optregret}.
\end{proof}

\subsubsection{Statement \& Proof in Terms of Stability Gap}

In this section we define a more refined metric for measuring stability for the setting of online convex optimization with memory. This metric measures the gap between the true loss and the surrogate (memory-less) loss incurred by an algorithm in a given interval, i.e. \emph{the price of past mistakes}. Formally, we define the \emph{stability gap} of an algorithm $\A$ which chooses actions $z_1, \ldots, z_s$ over an interval $I=[r, s]$ to be:

$$\stabgap_\A(I) = \sum_{t=r}^s \left|f_t(z_{t-H}, \ldots, z_t) - \tilde{f}_t(z_t)\right|$$

Note that this is a more general, but less intuitive, way to phrase our results since for coordinate-wise $L$-Lipschitz losses we can always bound $\stabgap_\A(I) \leq H^2 L \shift_\mA(I)$. However, this alternative formulation facilitates the use of newer results in the literature \cite{simchowitz2020making} of nonstochastic control (see section \ref{sec:drc_ons_apply}). In the rest of this section, we prove the following analogue of Theorem \ref{thm:meta} in terms of this quantity:

\begin{theorem}\label{thm:meta_v2} Let $f_t$ have range $[0, 1]$ and set $\sigma=1/T$. Algorithm \ref{alg:adaptive_reg} achieves the following bound for any experts $\{\mA_i\}$: over any interval $I = [r, s] \subseteq [T]$: 
\begin{equation*}
    \sum_{t=r}^s \E[f_t(z_{t-H:t})] \leq \gamma_\eta \min_{i\in [N]} \sum_{t=r}^s \tilde{f}_t(z_t^i) + \left(\dfrac{4}{\eta} + 3 \stabgap_\A(I)\right) \log(TN) ~.
\end{equation*}
where we denote $\stabgap_{\mA}(I) \doteq \max\limits_{i \in [N]} \stabgap_{\mA_i}(I)$ and $\gamma_\eta \doteq 1 + 4 \eta \stabgap_\A(I)$.
\end{theorem}

First note that lemmas \ref{lem:shrinking}, \ref{lem:adaptive}, \ref{lem:switches} remain unchanged and that lemma \ref{lem:memadregret} is no longer necessary. Hence we first prove an analogue of Lemma \ref{lem:stability} which bounds the stability gap of MARA in terms of the maximal stability gap of the experts.

\begin{lemma}\label{lem:stabgap}
For any $I=[r,s]$, we have that:

$$ \E[\stabgap_{\mara}(I)] \leq 2 \eta \stabgap_\A(I) \min_{i\in [N]}\sum_{t=r}^s\tilde{f}_t(x_t^i) + 3\,\stabgap_\A(I) \log(TN) $$

where we denote $\stabgap_{\mA}(I) \doteq \max\limits_{i \in [N]} \stabgap_{\mA_i}(I)$.
\end{lemma}

\begin{proof}
For every subinterval $J=[r_j, s_j]$ in which there is no expert switch, meaning $(z_{r_j}, \ldots, z_{s_j}) = (z^i_{r_j}, \ldots, z^i_{s_j})$ are chosen by the same expert $\mathcal{A}_i$, we can bound:
\begin{align*}
\sum_{t=r_j}^{s_j} \left| f_t(z_{t-H}, \ldots, z_t) - \tilde{f}_t(z_t) \right|&\leq H \cdot 1 + \sum_{t=r_j+H}^{s_j} \left| f_t(z_{t-H}, \ldots, z_t) - \tilde{f}_t(z_t) \right| & (f_t(\cdot)\in [0,1])\\
&= H + \sum_{t=r_j+H}^{s_j} \left| f_t(z^i_{t-H}, \ldots, z^i_t) - \tilde{f}_t(z^i_t) \right| & (\text{no switch}) \\
&\leq H + \stabgap_{\A_i}(J) \\
&\leq H + \stabgap_{\A}(I)
\end{align*}

Assuming there were $k_I$ switches and adding up the results for each of the subintervals, we have:
$$
\stabgap_{\mara} \leq k_I (H+ \stabgap_\A(I)) $$
Since the RHS is linear in $k_I$ we can take expectation over both sides and use Lemma \ref{lem:switches}  and the simplyfing assumption $\log(TN) \geq H\geq 2$ to obtain:
\begin{align*}
\E[\stabgap_{\mara}] &\leq 2 \eta \stabgap_\A(I) \min_{i\in [N]}\sum_{t=r}^s\tilde{f}_t(x_t^i) + 3\,\stabgap_\A(I) \log(TN) 
\end{align*}
\end{proof} Now we can give the proof of the main theorem:\begin{proof}[Proof of Theorem \ref{thm:meta_v2}] Using the following simplifying assumptions: $\eta \leq 1/2$ ($\Rightarrow 1/(1-\eta) \leq 2), \stabgap_\A(I) \geq 1$; and $2\leq H \leq \log (TN)$, we conclude that:
\begin{align*}
\sum_{t=r}^s \E[f_t(z_{t-H}, \ldots, z_t)] &\leq  \sum_{t=r}^s \E[\tilde{f}_t(z_t)] + \stabgap_{\mara}(I) \\
&\leq \dfrac{1}{1-\eta} \min_{i\in [N]} \sum_{t=r}^s \tilde{f}_t(z_t^i) + \dfrac{1}{(1-\eta)\eta} [\log(TN) + 2] + \stabgap_{\mara}(I) & (\text{Lemma } \ref{lem:adaptive}) \\
&\leq \gamma_\eta \min_{i\in [N]} \sum_{t=r}^s \tilde{f}_t(z_t^i) + \left(\dfrac{4}{\eta} + 3 \stabgap_{\A}(I)\right) \log(TN) & (\text{Lemma } \ref{lem:stabgap})
\end{align*}
\end{proof}

\subsection{Efficient Implementation} \label{sec:efficient} 
Algorithm \ref{alg:adaptive_reg}, as shown in Theorem \ref{thm:adaoco}, exhibits desirable performance in terms of adaptive regret for functions with memory. However, in terms of computational efficiency MARA (Algorithm \ref{alg:adaptive_reg}) with $N = T$ has overall complexity of $\Theta(T^2)$: this is considerably slower than efficient non-adaptive methods. In this section, we provide an efficient variant of MARA with nearly the same guarantees.

The idea behind the efficient implementation is to avoid the redundancy of all the $N = T$ experts that each exhibit good enough regret on a given interval $I = [r, s] \subseteq [T]$: it is enough to only consider a sparse working set $S_t$ of \emph{active} experts for all $t \in [T]$. The explicit construction of these sets (see section \ref{sec:working}) is irrelevant for the analysis so we simply state its useful properties below. 

\begin{claim}\label{claim:workingsets}
    The following properties hold for the working sets $S_t$ for all $t \in [T]$: (i) $|S_t| = O(\log T)$; (ii) $[s, (s+t)/2] \cap S_t \neq \emptyset$ for any $s \in [t]$; (iii) $S_{t+1} \backslash S_t = \{t+1\}$; (iv) $|S_t \backslash S_{t+1}| \leq 1$.
\end{claim}

The efficient implementation mimics Algorithm \ref{alg:adaptive_reg} with $N = T$ and the following expert algorithms: for each $i \in [N]$ define $b_i = \min \{ t \in [T], \text{ s.t. } i \in S_t\}$ and $e_i = \max \{ t \in [T], \text{ s.t. } i \in S_t \}$; take $\mA_i$ to be $\mA_{\text{scOGD}}$ run with initial point $z_0 \in \mathcal{K}$ on the interval $[b_i, e_i]$; let $\mA_i$ play the constant action $z_0$ outside that interval. A naive implementation of this algorithm still incurs $\Theta(T^2)$ complexity due to weight updates for all $i \in [N]$.

Our efficient algorithm, MARA-EFF (Algorithm \ref{alg:adaptive_eff}), instead each round $t \in [T]$ only updates the weights of the $i \in S_t$ {\it active} experts separately, and handles the {\it inactive} expert weights, that all play $z_0$, in a collective way. This ensures that MARA-EFF incurs $O(\log T)$ complexity each round $t \in [T]$ which is the computational overhead over efficient non-adaptive methods. Furthermore, this efficient version attains adaptive regret guarantees similar to those of Algorithm \ref{alg:adaptive_reg}, with an additional $O(\log T)$ multiplicative factor, as stated in the theorem below.

\begin{theorem}\label{thm:efficient}
Let $f_t$ be coordinate-wise $L$-Lipschitz with $H$ memory and range $[0, 1]$, assume the surrogate losses are strongly convex with bounded gradient norms. Then, with the choices of $\eta$ as in Theorem \ref{thm:adaoco}, MARA-EFF (Algorithm \ref{alg:adaptive_eff}) achieves the following adaptive regret bounds: over any $I = [r, s] \subseteq [T]$ interval,
\begin{equation}\label{eq:efflogfactor}
    \sum_{t=r}^s \E[f_t(z_{t-H:t})] \leq \left( 1 + \frac{1}{\log T} \right) \min_{z \in \mathcal{K}} \sum_{t=r}^s \tilde{f}_t(z) + O(H^2 L \log^4 T) ~,
\end{equation}
\begin{equation}\label{eq:effoptregret}
    \sum_{t=r}^s \E[f_t(z_{t-H:t})] \leq \min_{z \in \mathcal{K}} \sum_{t=r}^s \tilde{f}_t(z) + \tilde{\O}(H^2 L \cdot \sqrt{OPT}) ~.
\end{equation}
\end{theorem}

\begin{algorithm} 
\caption{Efficient Meta Adaptive Regret Algorithm (MARA-EFF)} \label{alg:adaptive_eff}
\begin{algorithmic}[1]
\STATE \textbf{Input:} action set $\mathcal{K}$, number of rounds $T$, online algorithm $\mathcal{A}$, parameters $\eta, \sigma$
\STATE \textbf{Setup:} $N = T $, arbitrary $z_0 \in \mathcal{K}$, $S_1 = \{1\}, U_1 = [N] \backslash S_1, V_1 = \emptyset$, assign weights $w_1^1 = w_1^u = 1, W_1^U = |U_1| w_1^u, W_1^V = 0$, feedback $\mathcal{F}_0^1 = \{z_0\}$, shrinking $q_1(i) = 0$ for all $i \in [N]$
\FOR{$t = 1, ..., T$} 
\STATE compute $z_t^i = \A(\mathcal{F}_{t-1}^i), \forall i\in S_t$ and $z_t^i = z_0, \forall i \not \in S_t$
\STATE denote $W_t = \sum_{i \in S_t} w_t^i + W_t^U + W_t^V$, and $p_t^i = \frac{w_t^i}{W_t}$ for $i \in S_t$, $p_t^u = \frac{w_t^u}{W_t}$, $p_t^v = \frac{W_t^V}{|V_t| W_t}$

\STATE choose $i_t = \begin{cases}
i \text{ w.p. } p_t^i = \frac{w_t^i}{W_t} &\text{if } i \in S_t, \\ 
i \text{ w.p. } p_t^u = \frac{w_t^u}{W_t} &\text{if } i \in U_t, \\
i \text{ w.p. } p_t^v = \frac{W_t^V}{|V_t| W_t} &\text{if } i \in V_t
\end{cases}$

\STATE shrink via $i_t = i_{t-1}$ w.p. $q_t(i_{t-1})$

\STATE play $z_t = z_t^{i_t}$, suffer loss $f_t(z_{t-H:t})$, observe $\tilde{f}_t(\cdot)$ and compute $\tilde{f}_t(z_0)$,  $\tilde{f}_t(z_t^i)$ for all $i \in S_t$
\STATE compute $\overline{w}_{t+1}^i = w_t^i e^{-\eta \tilde{f}_t(z_t^i)}, \forall i \in S_t$, $\overline{w}^u_{t+1} = w^u_t e^{-\eta \tilde{f}_t(z_0)}$, and $\overline{W}^V_{t+1} = W^V_t e^{-\eta \tilde{f}_t(z_0)}$
\STATE smooth $w_{t+1}^i = (1-\sigma)\overline{w}_{t+1}^i + \sigma \overline{W}_{t+1}/N, \forall i \in S_t$, and $w_{t+1}^u = (1-\sigma) \overline{w}^u_{t+1} + \sigma \overline{W}_{t+1} / N$
\STATE smooth $W^V_{t+1} = (1-\sigma) \overline{W}^V_{t+1} + \sigma |V_{t}| \overline{W}_{t+1} / N$

\STATE update $S_t \to S_{t+1}$, $U_t \to U_{t+1}$, $V_t \to V_{t+1}$, $W^U_{t+1} = |U_{t+1}| w_{t+1}^u$, and $w_{t+1}^{t+1} = w_{t+1}^u$
\STATE let $i_t^v = S_t \setminus S_{t+1}$ when $S_t \setminus S_{t+1} \neq \emptyset$, update $W_{t+1}^V = W_{t+1}^V + w_{t+1}^{i_t^v}$

\STATE denote $q_{t+1}(i_t=i) = \begin{cases}
w_{t+1}^i/w_t^i &\text{if } i \in S_t, \\ 
w_{t+1}^u/w_t^u &\text{if } i \in U_t, \\
(W_{t+1}^V-w_{t+1}^{i_t^v})/W_t^V &\text{if } i \in V_t
\end{cases}$

\STATE update $\mathcal{F}_t^i = \mathcal{F}_{t-1}^i \cup \{ z_t^i, \tilde{f}_t \}$ for $i \in S_{t+1} \cap S_t$, and $\mathcal{F}_t^i = \{z_0\}$ for $i \in S_{t+1} \setminus S_t$
\ENDFOR
\end{algorithmic}
\end{algorithm}

The efficient implementation is presented in detail in Algorithm \ref{alg:adaptive_eff}. Even though all inactive experts play action $z_0$, we differentiate between unborn $U_t = [N] \setminus [t]$ and dead $V_t = [t] \setminus S_t$ experts. We keep track of the weights of each of these groups denoted as $W_t^U, W_t^V$, with the unborn experts having all equal weights $w_t^u = W_t^U / |U_t|$. The shrinking probabilities are denoted as $q_t(i_{t-1})$, for example given by $q_t(i_{t-1}) = \frac{w_t^{i_{t-1}}}{w_{t-1}^{i_{t-1}}}$ in Algorithm \ref{alg:adaptive_reg}. We first show an equivalence between Algorithm \ref{alg:adaptive_eff} and a specific case of Algorithm \ref{alg:adaptive_reg}.

\begin{lemma}\label{lem:equiv}
Algorithm \ref{alg:adaptive_eff} is equivalent to Algorithm \ref{alg:adaptive_reg} with $N = T$ and the following $\mA_i$, for all $i \in [N]$: $\mA_i$ plays $\mA$ on the interval $[b_i, e_i]$; $\mA_i$ plays constant action $z_0$ outside of that interval.
\end{lemma}
\begin{proof}
Let $\tilde{w}_t$ be the weights by the specified version of Algorithm \ref{alg:adaptive_reg}. Given the corresponding weight initializations and updates in both algorithms, we note that for all $t \in [T]$,
\begin{equation*}
    w_t^u = \tilde{w}_t^i, \, \forall i \in U_t, \quad w_t^i = \tilde{w}_t^i, \, \forall i \in S_t, \quad W_t^V = \sum_{i \in V_t} \tilde{w}_t^i ~.
\end{equation*}
This is evident given that we update the weights in Algorithm \ref{alg:adaptive_eff} according to the hypothetical played actions of Algorithm \ref{alg:adaptive_reg}. To conclude equivalency, we need to additionally show that the shrinking technique yields its expected results, i.e. Lemmas \ref{lem:shrinking}, \ref{lem:switches}. This is done by showing the following identity,
\begin{equation}\label{eq:equiv_prop}
    \sum_{i=1}^N \Pr[i_{t-1}=i] \cdot q_t(i_{t-1}) = \frac{W_t}{W_{t-1}} ~.
\end{equation}
Note that Lemma \ref{lem:switches} follows immediately from \eqref{eq:equiv_prop}. In fact, we overestimate the number of switches in this case since we count switches between experts of $U_t, V_t$ which all play the same action. However, this overestimation does not negatively impact our results. On the other hand, one still needs to use induction to show Lemma \ref{lem:shrinking} simultaneously with \eqref{eq:equiv_prop}. Assume for $t-1$ that $\Pr[i_{t-1}=i]=p_{t-1}^i$ for $i \in S_{t-1}$, $\Pr[i_{t-1} = i] = p_{t-1}^u$ for $i \in U_{t-1}$ and $\sum_{i \in V_{t-1}}\Pr[i_{t-1}=i] = |V_{t-1}| p_{t-1}^v$ (true for $t=1$). The identity in \eqref{eq:equiv_prop} is shown as follows,
\begin{equation*}
    \sum_{i=1}^N \Pr[i_{t-1}=i] \cdot q_t(i_{t-1}) = \sum_{i \in S_{t-1}} \frac{w_t^i}{W_{t-1}} + |U_{t-1}| \cdot \frac{w_t^u}{W_{t-1}} + \sum_{i \in V_{t-1}} \frac{W_t^V - w_{t}^{i_{t-1}^v}}{ |V_{t-1}| W_{t-1}} = \frac{W_t}{W_{t-1}},
\end{equation*}
using the facts that $\sum_{i \in U_{t-1}} \tilde{w}_t^i = |U_{t-1}| w_t^u$ and $\sum_{i \in V_{t-1}} \tilde{w}_t^i = W_t^V - w_t^{i_{t-1}^v}$ by construction of $i_{t-1}^v$.
This concludes \eqref{eq:equiv_prop}, and following the proof of Lemma \ref{lem:shrinking}, we show the inductive step for $t$ using \eqref{eq:equiv_prop}. The stated equivalence follows given the weights are equivalent and the shrinking technique works as expected. 
\end{proof}

\begin{lemma}\label{lem:effregret}
    Suppose for any $I = [r, s] \subseteq [T]$, an online algorithm achieves the following regret if $r \in S_{t=s}$,
    \begin{equation*}
        \sum_{t \in I} \E[f_t(z_{t-H:t})] \leq \gamma \cdot \min_{z \in \mathcal{K}} \sum_{t \in I} \tilde{f}_t(z) + R,
    \end{equation*}
    where $\gamma, R$ are both independent of $I$. Then, for any interval $I = [r, s] \subseteq [T]$, it achieves adaptive regret
    \begin{equation*}
        \sum_{t \in I} \E[f_t(z_{t-H:t})] \leq \gamma \cdot \min_{z \in \mathcal{K}} \sum_{t \in I} \tilde{f}_t(z) + R (\log_2 |I| + 1) ~.
    \end{equation*}
\end{lemma}
\begin{proof}
    Let $2^k \leq |I| < 2^{k+1}$, we prove the lemma using induction on $k$. For $k=0$, $|I|=1$, so $r = s \in S_{t=s}$, and the lemma statement follows trivially. Suppose the statement holds for $k-1$, and let $2^k \leq |I| < 2^{k+1}$. According to Claim \ref{claim:workingsets}, there exists $i \in [r, (r+s)/2]$ such that $i \in S_{t=s}$. The case of $i=r$ is trivial so suppose $i > r$. This means that according to the assumption,
    \begin{equation*}
        \sum_{t =i}^s \E[f_t(z_{t-H:t})] \leq \gamma \cdot \min_{z \in \mathcal{K}} \sum_{t =i}^s \tilde{f}_t(z) + R ~.
    \end{equation*}
    Denote $I_1 = [r, i-1]$ with $|I_1| < |I|/2 < 2^k$, so the lemma statement holds for $I_1$ giving
    \begin{equation*}
        \sum_{t =r}^{i-1} \E[f_t(z_{t-H:t})] \leq \gamma \cdot \min_{z \in \mathcal{K}} \sum_{t =r}^{i-1} \tilde{f}_t(z) + R (\log_2 |I_1| + 1) ~.
    \end{equation*}
    Since $\log_2 |I_1| + 1 < \log_2 |I|$, summing up the two bounds above concludes the lemma proof.
\end{proof}

\begin{proof}[Proof of Theorem \ref{thm:efficient}]
Given the result from Lemma \ref{lem:equiv}, we prove the theorem for the specified version of Algorithm \ref{alg:adaptive_reg} with $N=T$ and the following $\mA_i$, for all $i \in [N]$: $\mA_i$ plays $\mA_{\text{scOGD}}$ on the interval $[b_i, e_i]$; $\mA_i$ plays constant action $z_0$ outside of that interval. The regret of each expert on the interval $[b_i, e_i]$ is $\reg_{\A_{\text{scOGD}}}(|[b_i, e_i]|) = O(\log T)$ , and the action shift upper bound is given by $\shift_{\A}(T) \leq \shift_{\A_{\text{scOGD}}}(T) + D \leq 2 \shift_{\A_{\text{scOGD}}}(T) = O(\log T)$. Hence we note that, according to Theorems \ref{thm:meta}, \ref{thm:adaoco}, given that $r \in S_{t=s}$ the condition in Lemma \ref{lem:effregret} holds with $\gamma = 1 + (\log T)^{-1}$ and $R = O(H^2 L \log^3 T)$ from \eqref{eq:logfactor}, and with $\gamma = 1$ and $R = \tilde{O}(H^2 L \cdot \sqrt{OPT})$ from \eqref{eq:optregret}. This means that the result from Lemma \ref{lem:effregret} concludes both \eqref{eq:efflogfactor}, \eqref{eq:effoptregret} statements of the theorem.
\end{proof}

\subsection{Lower bound}\label{sec:lower}

In this section, we show that it is impossible to achieve $o(\sqrt{T})$ adaptive regret for functions with memory. Recall that, assuming Lipschitz loss and strongly convex surrogate loss, $O(\log T)$ logarithmic {\it standard} regret is attained for functions with memory \citep{anava2015online}. Furthermore, for strongly convex losses adaptive (poly-)logarithmic regret is also attained via an experts scheme \citep{hazan2009efficient}. However, as we show next, these two settings cannot be combined together to get a bound $o(\sqrt{T})$. Furthermore, a consequence of Fact \ref{claim:cogd} implies $O(\sqrt{T})$ adaptive regret for functions with memory (see section \ref{sec:ogd}). This directly motivates the positive result in Section \ref{sec:ada_regret} interpreted as either a stronger first-order adaptive regret or poly-logarithmic adaptive regret with a diminishing multiplicative factor.

\begin{theorem}\label{thm:lower}
Assume the losses $f_t$ with $H\geq1$ memory are coordinate-wise Lipschitz, and their surrogate losses $\tilde{f}_t$ are strongly convex. For any online algorithm $\mathcal{A}$, the lower bound holds as
\begin{equation}\label{eq:lower_app}
    \adreg_{\A}(T)  = \sup_{f_1, \dots, f_T} \left(\sup_{[r, s] = I \subset [T]} \left[ \sum_{t \in I} f_t(z_{t:t-H}) - \min_{z \in \mathcal{K}} \sum_{t \in I} \tilde{f}_t(z) \right] \right) = \Omega(\sqrt{T}) ~.
\end{equation}
\end{theorem}

The key insight behind this lower bound is that in the standard OCO setting with no memory, adaptive regret guarantees necessarily mean a lot of movement between consecutive actions, i.e. large action shift. Consequently, this tradeoff between adaptive regret and action shift deters loss functions with memory to exhibit low adaptive regret despite the strong convexity assumption.

\begin{lemma}\label{lem:lower}
Fix an arbitrary $T$ number of rounds. Let algorithm $\A$ incur adaptive regret $\adreg_{\A}(T) = R$ over $f_1, \dots, f_T$ strongly convex losses. Then, as long as the regret is nontrivial $R = o(T)$, the action shift has the following lower bound $\shift_{\A}(T) = \Omega(T/R)$.
\end{lemma}

\begin{proof}
To prove the result stated in the lemma, we take on the role of the adversary and provide a sequence of strongly convex functions $f_1, \dots, f_T$ that guarantee $\Omega(T/R)$ action shift for any algorithm $\mathcal{A}$ with adaptive regret $R$. For simplicity, assume $\mathcal{K} = [-1, 1]$ since a general convex set can always be reduced to this case via simple transformations preserving strong convexity. Additionally, assume $R = R(T)$ is only a function of $T$ for simplicity: in general, using the constant $C$ in $R = C R(T)$ independent of $T$ instead of $1$ implies the same result. The losses are picked to drive the local minima to $\pm 1$ alternatively as detailed below.

Let $k = \lceil 4 R \rceil < T$ given $R = o(T)$ and $n = \lfloor T / k \rfloor$. We will have $n$ blocks of $k$ size each block with identical loss functions that alternate between blocks. In particular, for all $j \in [n]$ denote $I_j = [r_j, s_j]$ with $r_j = (j-1) \cdot k + 1, s_j = j \cdot k$. For all odd $j \in [n]$ we take $f_t(z) = (z-1)^2, \, \forall t \in I_j$ and for all even $j \in [n]$ we take $f_t(z) = (z+1)^2, \, \forall t \in I_j$. Note that $\min_{z \in \mathcal{K}} \sum_{t \in I_j} f_t(z) = 0$, hence the adaptive regret bound gives $\sum_{t \in I_j} f_t(z_t) \leq R$, implying $\min_{t \in I_j} f_t(z_t) \leq R/k$, for all $j \in [n]$ over the actions $\{z_t\}_{t \in [T]}$ of the algorithm $\mathcal{A}$.

For odd $j \in [n]$ we have that $f_t(z_t) = (z_t - 1)^2$ with $t \in I_j $ is a decreasing function in $\mathcal{K} = [-1, 1]$. Denote $m(j) = \argmax_{t \in I_j} z_t$ and note that $\min_{t \in I_j} f_t(z_t) = (z_{m(j)} - 1)^2 \leq R/k$ yielding the bound $z_{m(j)} \geq 1 - \sqrt{R/k} > 0$. Analogously, for even $j \in [n]$, denote $m(j) = \argmin_{t \in I_j} z_t$ to obtain the corresponding bound $z_{m(j)} \leq \sqrt{R/k} - 1 < 0$. Using both of these bounds along with triangle inequality, we get that
\begin{equation*}
    \forall 1 \leq j < n, \quad \sum_{t = m(j)}^{m(j+1)-1} \lvert z_{t+1}-z_t \rvert \geq \lvert z_{m(j+1)} - z_{m(j)} \rvert \geq 2 - 2 \sqrt{R / k} \geq 1,
\end{equation*}
since $4R \leq k < 4R+1$ by definition. We conclude the lemma statement by laying out the following facts: by definition $n \geq (T+1)/k - 1$ which implies $n-1 \geq (T+1)/(4R+1) - 2 = \Omega(T/R)$ given $R = o(T)$; the indices are trivially bounded $m(n) \leq T, m(1) \geq 1$.
\begin{equation*}
    \sum_{t=1}^{T-1} \lvert z_{t+1} - z_t \rvert \geq \sum_{t=m(1)}^{m(n)-1} \lvert z_{t+1} - z_t \rvert \geq n - 1 = \Omega(T/R),
\end{equation*}
which finishes the proof of the lower bound for action shift $\shift_{\A}(T) = \Omega(T/R)$.
\end{proof}

\begin{proof}[Proof of Theorem \ref{thm:lower}]
To show the lower bound for loss functions with memory, we construct losses that incorporate action shift and use the result from Lemma \ref{lem:lower}. In particular, consider $H = 1$ and define $f_t(z_t, z_{t-1}) = \tilde{f}_t(z_t) + \| z_t - z_{t-1} \|$ to be the loss function for all $t \in [T]$ with $\tilde{f}_t$ strongly convex surrogate losses. The adaptive regret for $f_t$ is lower bounded by the adaptive regret of $\tilde{f}_t$ given that $f_t(z, z') \geq \tilde{f}_t(z)$ for any $z, z'$. Moreover, the standard regret for $f_t$, which is always upper bounded by adaptive regret, is lower bounded by the action shift of the used algorithm $\mathcal{A}$ due to the given choice of $f_t$ (this can be ensured by scaling values of $\tilde{f}_t$ so that even negative regret doesn't cancel out action shift). The case when the adaptive regret of $\tilde{f}_t$ is $\Omega(T)$ is trivial, so assume the algorithm used achieves adaptive regret $o(T)$. Finally, use Lemma \ref{lem:lower} to obtain
\begin{equation*}
    \adreg_{\A}(T) \geq \max \left( R, \Omega(T/R) \right) \geq \frac12 \left( R + \Omega(T/R) \right) = \Omega(\sqrt{T}),
\end{equation*}
which concludes the theorem statement.
\end{proof}

\section{Adaptive Regret for Control of Time-Varying Dynamics}

\subsection{Online Control over Time-Varying Dynamics}\label{sec:control-full}

In this section we give the complete and formal statement of our main theorem, along with its proof. For simplicity, assume the norm bound on the disturbances is given by $W=1$ to avoid carrying the term $W$ as a constant.

\begin{theorem}\label{thm:meta_controller}
Suppose the controllers $\C_i$ with $N = T$ are valid base controllers (i.e. they satisfy the properties of Definition \ref{def:base_controller}), and for each $i \in [N]$, letting $\mathcal{C}_i$ start a control algorithm $\mathcal{C}$ from round $t=i$ while playing $u^i_t=0$ for $t<i$.\footnote{This is replaced by a stabilizing controller for stabilizable systems. This is in effect the same thing since the stabilizing controller $K_t$ can be absorbed into $A_t$ and the procedure reduces to playing $u=0$ on the surrogate system $(A_t+B_t K_t, B_t)$} Then, with the appropriate choices of $\eta$ (see Theorem \ref{thm:adaoco}) and $\sigma = 1/T$, MARC (Algorithm \ref{alg:adaptive_control_main}) achieves the following adaptive regret with respect to $\Pi$ class of policies: over any $I = [r, s] \subseteq [T]$ interval
\begin{equation}\label{eq:controllog}
    \sum_{t=r}^s \E[c_t(x_t, u_t)] \leq \left( 1 + \nu_T \right) \min_{\pi \in \Pi} \sum_{t=r}^s c_t(x_t^{\pi}, u_t^{\pi}) + O\left(\reg_{\mathcal{C}}(I)\right) + \tilde{O}\left(H^2 L \shift_{\mathcal{C}}(T)\right) + 2 |I| \eps ,
\end{equation}
where the diminishing factor is $\nu_T = (\log T)^{-1}$. Furthermore, we also get
\begin{equation}\label{eq:controlopt}
    \sum_{t=r}^s \E[c_t(x_t, u_t)] \leq \min_{\pi \in \Pi} \sum_{t=r}^s c_t(x_t^{\pi}, u_t^{\pi}) + O\left(\reg_{\mathcal{C}}(I)\right) + \tilde{O}\left(H^2 L \shift_{\mathcal{C}}(T) \sqrt{OPT} \right) + 2 |I| \eps ~,
\end{equation}
where $OPT = \min\limits_{\pi \in \Pi} \sum\limits_{t=1}^T c_t(x_t^{\pi}, u_t^{\pi})$ is the cost of the best policy in hindsight.
\end{theorem}

\begin{remark}\label{remark:sg} {(Stability Gap)}
Analogues of the above results can be straightforwardly derived in terms of the more general \emph{stability gap} quantity by employing Theorem \ref{thm:meta_v2} instead.
\end{remark}

\begin{remark}{(Efficient implementation)} Note that with $N=T$ base controllers we can use the procedure described in section \ref{sec:efficient} for an efficient implementation that keeps track of only $O(\log T)$ active controllers. The inactive ones are represented by the stationary $u=0$ (or stabilizing) controller. This efficient version incurs only a $O(\log T)$ extra multiplicative adaptive regret factor due to Lemma \ref{lem:effregret}. Hence, our results are also computationally efficient (indeed, this is the version we implement in section \ref{sec:experiments}).
\end{remark}

\begin{proof}[Proof of Theorem \ref{thm:meta_controller}]
Given that all controllers $\C_i$ output $u_{1:T}^i$ action sequences from $\mathcal{U}_{1:T}$, then $u_t^i \in \mathcal{U}_t$ for all $i \in [N]$ and $t \in [T]$. This implies that $u_t \in \mathcal{U}_t$ for all $t \in [T]$, so $u_{1:T} \in \mathcal{U}_{1:T}$. Thus, properties (i), (ii) of Definition \ref{def:base_controller} hold for $\mathcal{C}_{\text{MARC}}$ too. The proxy cost function $f_t$ is coordinate-wise $L$-Lipschitz. Since the domain of the cost functions $c_t$ is $[0, 1]$, so is the domain of the surrogate losses $\tilde{f}_t$. Algorithm \ref{alg:adaptive_control_main} uses the surrogate function $\tilde{f}_t(\cdot)$ for updating weights exactly as Algorithm \ref{alg:adaptive_reg}, so the Theorem \ref{thm:meta} result holds for proxy cost $f_t$. Next, use the $(H, \eps)$-bounded memory property of $\mathcal{U}_{1:T}$ to transfer the proxy cost results to the control cost $c_t(\cdot, \cdot)$ as follows,
\begin{equation*}
    \sum_{t=r}^s \E[c_t(x_t, u_t)] \leq \sum_{t=r}^s \E[c_t(\hat{x}_t, u_t)] + |I| \eps  = \sum_{t=r}^s \E[f_t(u_{t-H:t})] + |I| \eps ~.
\end{equation*}
For each $i \in [N]$, use \eqref{eq:memadregret1} from Lemma \ref{lem:memadregret} and the $(H, \eps)$-bounded memory property to obtain
\begin{equation*}
    \sum_{t=r}^s \tilde{f}_t(u_t^i) \leq \sum_{t=r}^s c_t(x^i_t, u^i_t) + H^2 L \shift_{\mathcal{C}}(T) + |I| \eps ~.
\end{equation*}
The controller $\mathcal{C}_r$ suffers regret $\reg_{\mathcal{C}}(I)$ over the interval $I = [r, s]$ since it starts running controller $\mathcal{C}$ from round $t=r$. We conclude both \eqref{eq:controllog} and \eqref{eq:controlopt} by applying the bounds derived above to the result from Theorem \ref{thm:meta}, with the corresponding parameter choices (and simplifying assumptions) as detailed in Theorem \ref{thm:adaoco} and its proof.
\end{proof}

\subsection{Reference Policy Classes} \label{sec:policyclasses}

We now describe formally a few classes of policies $\Pi$ that our controllers compete with. 
The most elementary policy classes for control of linear dynamical systems are linear policies, i.e. those that choose a control signal as a (fixed) linear function of the state. Formally, 
\begin{definition}[Linear State Feedback Policies] 
A linear (state feedback) policy $K \in \reals^{d_u \times d_x}$ is parameterized by a linear operator mapping state to control.
\end{definition}

The most important subcase of linear policies are those that are stable, meaning that the signal they create does not cause an unbounded state. The following definition extends stability and stabilizability to the time-varying dynamics case. 
\begin{definition} [Sequentially Stabilizing] \label{def:seqstab}
For a time-variant LDS we say that a sequence of linear policies $\{K_t\}$ is $(\kappa,\delta)$ \emph{sequentially stabilizing} iff 
$$ \forall \, I=[r,s] \in [T] \ , \ \left\| \prod_{t=s}^r (A_t + B_t K_t) \right\|  < \kappa (1-\delta)^{|I|} . $$
We say that a single linear policy $K$ is $(\kappa,\delta)$ sequentially stabilizing if the above holds for $K_t = K$.
\end{definition}

We denote by $\lin^{\kappa, \delta}$ the set of all linear policies bounded in norm by $ \| K \| \leq \kappa $, that are $(\kappa, \delta)$ sequentially stablizing for our changing LDS. 

The next policy class is that of Linear Dynamical Controllers. These simulate an ``internal linear dynamical system" so to recover a hidden state, and play a linear function over this hidden state. The formal definition is given below.  

\begin{definition}[Linear Dynamic Controllers] A linear dynamic controller $\pi$ has parameters  $\left(A^{\pi}, B^{\pi}, C^{\pi}, D^{\pi}\right)$, and chooses a control at time $t$ according to 
$$
u_{t}^\pi = \C^{\pi} s^\pi_{t}+ D^{\pi} x^\pi_{t} \ ,  \ s^\pi_{t+1}= A^{\pi} s^\pi_{t}+B^{\pi} x^\pi_{t}.
$$

where $s_t^\pi$ is the internal state, $x_t^\pi$ is the input, $u_t^\pi$ is the output 
\end{definition}

Clearly the class of policies LDC approximates linear policies without any error, since we can take $A^\pi, B^\pi,C^\pi$ to be zero, and get a linear policy. 

We denote by $\ldc^{\kappa, \delta}$ the class of all LDC policies with parameters that satisfy
$$ \|A^\pi\| \leq 1-\delta \ , \ (\|B^\pi\|+\|C^\pi\|+ \|D^\pi\|)^2 \leq \kappa . $$ 

LDCs are considered to be state-of-the-art in control of LDS, even for systems with full observation. We proceed to define even more powerful policy classes proposed by the nonstochastic control literature. 

If the state $x_t$ is fully observable, then the learner can compute the disturbance after each timestep, according to $w_t = x_{t+1} - A_t x_t - B_t u_t$ , where $u_t$ is the action taken in the environment. This has given rise in the popular adoption of disturbance action controllers in the literature (see \citet{agarwal2019online}, \citet{agarwal2019logarithmic}, \citet{hazan2019nonstochastic} for example), as defined below:

\begin{definition}{(Disturbance Action Controllers)}
A disturbance action controller is parametrized by a sequence of $H$ matrices $M = \{M^{[i]}\}_{i=1}^{H}$ and a sequence $K_t$ of $(\kappa, \delta)$ sequentially stabilizing controllers, acting according to $$u_t = K_t x_t + \sum\limits_{i=0}^{H-1} M^{[i+1]} w_{t-i}~.$$\end{definition}
We denote by $\dac^{H,\gamma, \{K_t\}}$ the set of all disturbance action controllers for given\footnote{The sequence can be given iteratively, for example it can be computed online based on iteratively revealed $(A_t, B_t)$.} sequentially stabilizing sequence $K_{1:T}$ and with $H$-length parametrization which satisfies $\sum_{i=1}^H \| M^{[i]}\| \leq \gamma $.


Finally, we consider disturbance response controllers as introduced by \citet{simchowitz2020improper}. This is similar conceptually to DAC, but acts upon a slightly different parametrization called nature's $x$ ($x^{\text{nat}}$) which represent the state the controller would have reached in the absence of exogenous control input. Concretely, given knowledge of the system up to the current moment, one can extract $x^{\text{nat}}_{t+1} = \sum\limits_{i=0}^{t} \phi_{t,i} w_{t-i}$ where $\phi_{t,i} \doteq \prod\limits_{j=0}^{i-1}(A_{t-j} + B_{t-j} K_{t-j})$. 

\begin{definition}{(Disturbance Response Controller)}
A disturbance response controller is parameterized by a sequence of $H$ matrices $M = \{M^{[i]}\}_{i=1}^{H}$ and a sequence $K_t$ of $(\kappa, \delta)$ sequentially stabilizing controllers, acting according to $$u_t = K_t x_t + \sum\limits_{i=0}^{H-1} M^{[i+1]} x^{\text{nat}}_{t-i}~.$$
\end{definition}

We denote by $\drc^{H,\gamma, \{K_t\}}$ the set of all disturbance action controllers for given stabilizing sequence $K_{1:T}$ and with $H$-length parametrization which satisfies $\sum_{i=1}^H \| M^{[i]}\| \leq \gamma $. 

While we focus on the full observation setting and introduce DRC mainly to formalize using the result of \citet{simchowitz2020making} as a base controller, we note that DRC can be particularly useful for partially observable systems, when employed with a parametrization based on nature's observations ($y^{\text{nat}}$) rather than states ($x^{\text{nat}}$). Our work thus can be extended to the case of partially observable systems.

We conclude this section by formalizing that over LTI systems competing with $\dac$/$\drc$\footnote{For presentation purposes, we often suppress the superscripts of the policy classes since these quantities are treated as constants and do not directly relate to the bulk of our contributions.} implies competing with $\lin$. This is not an empty statement, since the $K_t$ terms for DAC \& DRC are only assumed to be stabilizing, and not necessarily optimal for the actual sequence. First, let us give a formal method of comparing these classes. 

\begin{definition}
We say that a policy class $\Pi_1$ $\eps$-approximates class $\Pi_2$ if the following holds: for every time-varying LDS $\{(A_t, B_t)\}$, every sequence of $T$ disturbances and cost functions, and every $\pi_2 \in \Pi_2$, there exists $\pi_1 \in \Pi_1$ such that, 
$$ \sum_{t=1}^T \left| c_t(x_t^{\pi_1},u_t^{\pi_1}) - c_t(x_t^{\pi_2}, u_t^{\pi_2})  \right| \leq T \eps .$$
\end{definition}

\begin{lemma}\label{lemma:approx-dac-lin}
The class $\dac^{H,\gamma, \{K_t\}}$ $\eps$-approximates the class $\lin^{\kappa,\delta}$ over LTI dynamics if the memory length satisfies  
$H = \Omega\left( \log \left( 1/\eps \right) \right)$.
\end{lemma}

\begin{lemma}\label{lemma:approx-drc-lin}
The class $\drc^{H,\gamma, \{K_t\}}$ $\eps$-approximates the class $\lin^{\kappa,\delta}$ over LTI dynamics if the memory length satisfies  
$H = \Omega\left( \log \left( 1/\eps \right) \right)$.
\end{lemma} 

The two lemmas above suggest that, over \emph{LTI} systems, regret results w.r.t. the policy classes $\dac$ and $\drc$ imply the same guarantees over linear $\lin$ policies. Similar results can be shown for $\ldc$ (see \cite{hazan2020lecture}).

\subsection{Deriving Adaptive Controllers}\label{sec:apply_control}

The adaptive controller results stated in section \ref{sec:control} hold with any base controller $\mathcal{C}$ that satisfies the requirements in Definition \ref{def:base_controller}. In this section, we show that two already existing control algorithms can be used in place of $\mathcal{C}$ to construct an explicit instantiation of $\marc$ and hence achieve specific adaptive regret guarantees. 

Of course, as in the example provided in the main text, we can just immediately use previous results in nonstochastic control for fixed systems to give implications for specific time-varying systems (e.g. $k$-switching systems). However, we prove that previous algorithms still exhibit the desired properties and regret bounds over changing dynamics under the generalized assumptions put forth. This yields much more general adaptive regret results.

\subsubsection{MARC-GPC}

The first relevant controller we consider is the GPC algorithm \citep{agarwal2019online} that achieves $\sqrt{T}$ regret over general convex costs under fully adversarial noise for {\it fixed} dynamics. It uses the DAC policy class parameterization and the OGD algorithm for OCO with memory to obtain the aforementioned result. The algorithm is spelled out for LTV dynamics below in Algorithm \ref{alg:GPC}.

\begin{algorithm}[H]
\begin{algorithmic}[1] 
\STATE {Input:} $(A_1, B_1)$, $H$, $\{\eta_t\}$, initialization $M_1^{[1:H]}$, constraint set $\mathcal{M}$ for parameters $M^{[1:H]}$.
\FOR{$t$ = $1 \ldots T$}
        \STATE  compute an operator of $\delta$ sequentially stabilizing linear controller $K_t$ for $(A_t, B_t)$
        \STATE  $\mbox{choose action }u_t = K_t x_t + \sum_{i=1}^{H} M_t^{[i]} w_{t-i} $
        \STATE  observe new state \& system $x_{t+1}, (A_{t+1}, B_{t+1})$
        \STATE compute disturbance $w_t = x_{t+1} - A_t x_t - B_t u_t$
        \STATE let $\tilde{f}_t(M^{[1:H]}) = c_t(\hat{x}_t(M^{[1:H]}), u_t(M^{[1:H]}))$
       \STATE update $M_{t+1}^{[1:H]} \leftarrow \text{Proj}_{\mathcal{M}} \left( M_{t}^{[1:H]} - \eta_t \nabla \tilde{f}_t(M_t^{[1:H]}) \right)$
\ENDFOR
 \caption{Gradient Perturbation Controller (GPC)}\label{alg:GPC}
\end{algorithmic}
\end{algorithm}

\begin{claim}\label{claim:appl_gpc}
The GPC algorithm from \citet{agarwal2019online}, Algorithm \ref{alg:GPC} is a valid base controller (in the sense of Definition \ref{def:base_controller}) with $H = \O(\log(1/\eps))$ and $L = \O(H)$. 
\end{claim}

Given the claim above, we can use GPC as a base controller for MARC to state adaptive regret results for general convex costs. First, we show, via a similar derivation to previous work, that GPC attains an analogous regret bound over time-varying dynamics.

\begin{theorem}\label{thm:gpc_time-varying}The GPC algorithm from \citet{agarwal2019online}, Algorithm \ref{alg:GPC}, over time-varying dynamics $(A_t, B_t)$ satisfies the regret bound,
\begin{equation}\label{eq:gpc_regret}
   \sum_{t=1}^{T} c_t (x_t, u_t) - \min_{\pi \in \Pi_{\text{DAC}}} \sum_{t=1}^{T} c_t (x^\pi_t,  u^\pi_t)  \le O(\sqrt{T}) ~.
\end{equation}
Furthermore, it also achieves $\shift_\text{GPC}(T) = O(\sqrt{T})$.
\end{theorem}

Henceforth, we state our {\it adaptive} regret result for MARC with base controller GPC as a straightforward application of Theorem \ref{thm:meta_controller}. 

\begin{corollary}\label{cor:marc-gpc}
Our meta-algorithm MARC with base controller $\mathcal{C} = \mathcal{C}_{\text{GPC}}$ achieves the following adaptive regret under fully adversarial noise and over general convex costs $c_t$:
\begin{equation}\label{eq:corgpc}
    \forall I = [r, s] \subseteq [T], \quad \sum_{t=r}^s \E[c_t(x_t, u_t)] \leq \left( 1 + \frac{1}{\log T} \right) \min_{\pi \in \dac} \sum_{t=r}^s c_t(x_t^{\pi}, u_t^{\pi}) + \tilde{O}(\sqrt{T}) ~.
\end{equation}
\end{corollary}

\begin{proof}
Use Theorem \ref{thm:meta_controller} with $\eps = T^{-1}$, and in particular the statement given in \eqref{eq:controllog}. Given Claim \ref{claim:appl_gpc}, the quantities $H, L$ are polylogarithmic in $T$. Finally, use Theorem \ref{thm:gpc_time-varying} to conclude the corollary.
\end{proof}

\begin{remark}
We remark that Algorithm \ref{alg:GPC}, the time-varying analogue of GPC, attains $O(\sqrt{T})$ {\it adaptive} regret due to Fact \ref{claim:cogd} without the multiplicative factor $(\log T)^{-1}$ relative to the corollary. The $\tilde{O}(\sqrt{T})$ component in \eqref{eq:corgpc} has the global value $T$ due to the parameter choice of $\eta$ that depends on global action shift, otherwise that component is $O(\sqrt{|I|})$ in terms of regret. However, we note that our meta-algorithm MARC is useful given the next remark and subsection.
\end{remark}

\begin{remark}
A similar derivation based on \citet{agarwal2019logarithmic} shows that, in the presence of stochastic noise, the proxy loss is strongly convex w.r.t. the GPC parameterization. Hence, in this setting, GPC achieves (poly-)logarithmic regret and action shift which means MARC-GPC achieves $\tilde{O}(\sqrt{OPT})$ adaptive regret.
\end{remark}

\subsubsection{MARC-DRC-ONS} \label{sec:drc_ons_apply}
The most recent and most general {\it fast rate} regret bounds in nonstochastic control can be found in \citet{simchowitz2020making}. In this work, the DRC-ONS control algorithm $\mathcal{C}_{\text{DRC-ONS}}$ achieves poly-logarithmic regret with strongly convex costs under fully adversarial noise for {\it fixed} dynamics and we apply it as a base controller for $\marc$. As the name suggests, $\mathcal{C}_{\text{DRC-ONS}}$ uses the DRC policy class to choose actions and a newly devised \mbox{Semi-ONS} method for OCO with memory to obtain the stated guarantee. The algorithm is spelled out for time-varying dynamics below in Algorithm \ref{alg:DRC-ONS}.

\begin{algorithm}[H]
\begin{algorithmic}[1] 
\STATE {Input:} $(A_1, B_1)$, $H$, $\eta$, initialization $M_1^{[1:H]}$, constraint set $\mathcal{M}$ for parameters $M^{[1:H]}$.
\FOR{$t$ = $1 \ldots T$}
        \STATE  compute an operator of $\delta$ sequentially stabilizing linear controller $K_t$ for $(A_t, B_t)$
        \STATE  choose action $u_t = K_t x_t + \sum_{i=1}^{H} M_t^{[i]} x^{\text{nat}}_{t-i} $
        \STATE  observe new state \& system $x_{t+1}, (A_{t+1}, B_{t+1})$
        \STATE compute disturbance $w_t = x_{t+1} - A_t x_t - B_t u_t$
        \STATE let $f_t(M_{t-H:t}^{[1:H]}) = c_t(\hat{x}_t(M_{t-H:t-1}^{[1:H]}), u_t(M_t^{[1:H]}))$
       \STATE update $M_{t+1}^{[1:H]} \leftarrow \mbox{Semi-ONS} \left(M_{t}^{[1:H]}, f_t(\cdot), \eta \right)$
\ENDFOR
 \caption{Disturbance Response Control via Online Newton Step (DRC-ONS)}\label{alg:DRC-ONS}
\end{algorithmic}
\end{algorithm}

\begin{claim}\label{claim:appl_ons}
The DRC-ONS algorithm from \citet{simchowitz2020making}, Algorithm \ref{alg:DRC-ONS} is a valid base controller (in the sense of Definition \ref{def:base_controller}) with $H = O(\log(1/\eps))$ and $L = O(\sqrt{H})$.
\end{claim}

The claim above suggests that we can use $\mathcal{C}_{\text{DRC-ONS}}$ as a base controller for \marc yielding adaptive regret results. Beforehand, we show that the DRC-ONS algorithm \citep{simchowitz2020making} as given in Algorithm \ref{alg:DRC-ONS} achieves regret bounds analogous to original work over {\it time-varying} dynamical systems.
\begin{theorem}\label{thm:drc-ons}
The DRC-ONS algorithm (Algorithm \ref{alg:DRC-ONS}) over time-varying dynamics $(A_t, B_t)$ satisfies the regret bound,
\begin{equation}\label{eq:drc_regret}
    \sum_{t=1}^{T} c_t (x_t, u_t) - \min_{\pi \in \Pi_{\text{DRC}}} \sum_{t=1}^{T} c_t (x^\pi_t,  u^\pi_t)  \le O(\text{poly}(\log T)) ~.
\end{equation}
Furthermore, it also achieves $\mathcal{SG}_{\text{DRC-ONS}}(T) = O(\text{poly}(\log T))$.
\end{theorem}
Note that even though this theorem statement applies for time-varying dynamics, its guarantee is not particularly meaningful in our setting. The guarantee suggests that DRC-ONS only competes with the single fixed best policy in hindsight over the whole interval $[1, T]$, which can incur linear cost itself given the potential environment changes. Hence, we finally state our {\it adaptive} regret result for \marc with base controller DRC-ONS, a nearly optimal result in this setting.

\begin{corollary}\label{cor:ons}
Our meta-algorithm \marc with base controller $\mathcal{C} = \mathcal{C}_{\text{DRC-ONS}}$ achieves the following adaptive regret under fully adversarial noise and over strongly convex costs $c_t$:
\begin{equation}\label{eq:corons}
    \forall I = [r, s] \subseteq [T], \quad \sum_{t=r}^s \E[c_t(x_t, u_t)] \leq \min_{\pi \in \Pi_{\text{DRC}}} \sum_{t=r}^s c_t(x_t^{\pi}, u_t^{\pi}) + \tilde{O}( \sqrt{OPT}) ~,
\end{equation}
where $OPT = \min\limits_{\pi \in \Pi_{\text{DRC}}} \sum_{t=1}^T c_t(x_t^{\pi}, u_t^{\pi})$ is the cost of the best policy of $\Pi_{\text{DRC}}$ in hindsight.
\end{corollary}

\begin{proof}
Use Theorem \ref{thm:meta_controller} with $\eps = T^{-1}$, and in particular the statement given in \eqref{eq:controlopt}. Given Claim \ref{claim:appl_ons}, the quantities $H, L$ are polylogarithmic in $T$. Finally, use Theorem \ref{thm:drc-ons} along with Remark \ref{remark:sg} to conclude the corollary.
\end{proof}

\subsection{Deferred Proofs in Control}

\subsubsection{Proofs for Policy Classes}\label{sec:policy_classes_pproofs}
We show the results for $\dac$ and $\drc$ on stable LDS with $K = 0$. The same holds for general LDS and $K$ since we can derive the results below over the system $(\tilde{A}, B)$ with $\tilde{A} = A + B K$ instead.
\begin{proof}[Proof of Lemma \ref{lemma:approx-dac-lin}]
Via the definition of a linear controller, for any $\pi \in \lin^{\kappa, \delta}$, we have
$$ u_t^{\pi} =  K x_{t}^{\pi} . $$
We proceed to show that every state can be approximated using the previous disturbances as follows:
\begin{align*}
u_{t+1}^{\pi} &= K x_{t+1}^{\pi} \\
& = K(A x_t^{\pi} + B u_t^{\pi} + w_t) \\
& = K (A + B K) x_t^{\pi} + K w_t \\
& =  \sum\limits_{i=0}^t K \left[ \prod_{j=0}^{i-1} (A + B K) \right]  w_{t-i} \\
& =  \sum_{i=0}^{H-1} K \left[ \prod_{j=0}^{i-1} (A + B K) \right]  w_{t-i} + Z_t , 
\end{align*}

where $Z_t = \sum_{i=H}^t K \left[ \prod_{j=0}^i (A_{t-j} + B_{t-j} K) \right]  w_{t-i} $. It remains to bound the magnitude of this residual term, 

\begin{align*}
\|Z_t\| & \leq \left \| \sum_{i=H}^{t} K  \left[ \prod_{j=0}^{i-1} (A + B K) \right]  w_{t-i} \right\| \\
& \leq  \sum_{i=H}^{t} \kappa (1-\delta)^{i}  \left \| w_{t-i} \right \| \\
& \le \kappa \sum_{i=H}^{\infty} (1-\delta)^{i} \\
&\leq \kappa \int_{i=H}^\infty e^{- \delta i} \\
& = \kappa \frac{1}{\delta} e^{-\delta H} \\
&\le \eps,
\end{align*}
where the last line follows by our choice of $H = \Omega\left( \frac{1}{\delta} \log \left(\frac{\kappa }{\delta \epsilon}\right)\right)$.
\end{proof}

\begin{proof} [Proof of Lemma \ref{lemma:approx-drc-lin}]
Analogous to the proof of Lemma \ref{lemma:approx-dac-lin}, we can express the action $u_{t+1}^\pi = K x_{t+1}^{\pi}$, as well as $x_{t+1}^{\text{nat}}$, in terms of the disturbances $w_t$ as follows
\begin{align}
u_{t+1}^{\pi} &= \sum_{i=0}^t K \left(A + B K\right)^i w_{t-i} ~. \label{eq:drc_upi} \\
x_{t+1}^{\text{nat}} &= \sum_{i=0}^t A^i w_{t-i} ~. \label{eq:drc_xnat}
\end{align}

For the choice of parameters $M^{[0]} = K$ and $M^{[i]} = K \left(A + B K\right)^i B K$, one can see that
\begin{equation*}
    \sum_{i=0}^{t-1} M^{[i]} x_{t-i}^{\text{nat}} = u_t^{\pi}
\end{equation*}
by equating the coefficients for each disturbance $w_i$ for $i \in [0, t-1]$. Denote $u_t^M = \sum_{i=0}^{H-1} M^{[i]} x_{t-i}^{\text{nat}}$ and notice that
\begin{equation*}
    u_t^{\pi} = u_t^{M} + \sum_{i=H}^{t-1} M^{[i]} x_{t-i}^{\text{nat}} ~.
\end{equation*}
Notice that according to \eqref{eq:drc_xnat} the quantity $x_{t+1}^{\text{nat}}$ for any $t$ has a norm bounded by $\sum_{i=0}^{\infty} (1-\delta)^i = \frac{1}{\delta}$ since $\|w_t\| \leq 1$. It follows that
\begin{equation*}
    \| u_t^{\pi} - u_t^{M} \| \leq \left( \sum_{i=H}^{t-1} M^{[i]} \right) \cdot \frac{1}{\delta} \leq C \sum_{i=H-1}^{\infty} (1-\delta)^i \leq C' \cdot e^{-
    \delta (H-1)} \leq \eps,
\end{equation*}
where the last line follows by our choice of $H = \Omega(\log (1/\eps))$.
\end{proof}

\subsubsection{Control Action Parametrization}
In this section, we describe a certain generic control setup, due to \citet{agarwal2019online, agarwal2019logarithmic, hazan2019nonstochastic,simchowitz2020improper,simchowitz2020making}, used to derive the results in section \ref{sec:apply_control}. We show that the properties of Definition \ref{def:base_controller} are satisfied by the relevant control algorithms \citep{agarwal2019online, simchowitz2020making}. Furthermore, we prove that these algorithms attain regret bounds analogous to their original work in the {\it time-varying} setting as well. Before proving these claims, we elaborate on the notion of control action {\it reparameterization}. Section \ref{sec:setting_control} defines action set sequence $\mathcal{U}_{1:T}$ as a tool to characterize the action sequence $u_{1:T}$ taken by a controller $\C$. Control reparameterization motivates this choice in the following way.

For a controller $\C$ with actions $u_{1:T} \in \mathcal{U}_{1:T}$, its reparameterization can be denoted as $u_t = u_t^M = g_t(M)$ for all $t \in [T]$, where $g_t(\cdot)$ is a function depending on the system specifications and independent of the instantiation of $\C$; on the other hand, $M$ is the parameter controlled by $\C$. Let $\C$ have control reparameterization such that $\mathcal{U}_t = \{ u_t \text{ s.t. } u_t = u_t^M = g_t(M), \,  \forall M \in \mathcal{M} \}$ for all $t \in [T]$. All the controller has to specify are the parameters $M_1, \dots, M_T$ from a constraint set $\mathcal{M}$ with the objective to minimize regret. Note that the policy classes DAC, DRC, and even linear policies, from section \ref{sec:policyclasses} follow this pattern.

Note that the Lipschitz property of the proxy cost and the action shift of the controller from Definition \ref{def:base_controller} can, in general, be defined with respect to the parameter $M$ of the controller $\C$. This encompasses the current definition with the parameter $M_t = u_t$ simply being the control action itself. Furthermore, all the results hold identically with this more general definition,
\begin{equation*}
    f_t(u_{t-H:t}) = f_t^M(M_{t-H:t}), \quad \shift_{\C}(T) = \sup_{c_1, \dots, c_T} \sum_{t=1}^{T-1} \| M_{t+1}-M_t \| ~.
\end{equation*}
This is evident since the coordinate-wise Lipschitz property and action shift are used in tandem to attain adaptive regret results (see proofs of Theorem \ref{thm:meta}, \ref{thm:meta_controller}, and in particular Lemma \ref{lem:memadregret}). Next, we prove the results stated in section \ref{sec:apply_control} linking to (and using) derivations analogous to the corresponding papers.

\subsubsection{Proofs for Time-Varying GPC}\label{sec:gpc-proof}

The results for the time-varying analogue of GPC given in Algorithm \ref{alg:GPC} hold under the assumption model of \citet{agarwal2019online}. In particular, suppose Assumption 3.1, 3.2 from \citet{agarwal2019online} are satisfied. Since we are dealing with time-varying dynamics, let Assumption 3.1 hold for all $t \in [T]$, i.e. the system matrices $(A_t, B_t)$ all have bounded spectral norm for all $t$.

\begin{proof}[Proof of Claim \ref{claim:appl_gpc}]
    First, note that the quantities $(A+BK)^i$ for fixed and $\prod_{j}^{j+i-1}(A_j + B_j K_j)$ for interval $[j, j+i-1]$ for time-varying dynamics have analogous norm bounds. As an intermediate step in the proof, let us argue that the magnitude bound, given by Lemma 5.5 in \citet{agarwal2019online}, still holds for time-varying dynamics, i.e. both the state and control are bounded for all $M \in \mathcal{M}$. This is true since all the bounds in that lemma, such as Lemma 5.4, follow identically by substituting $(A+BK)^i$ with $\prod_{j}^{j+i-1}(A_j + B_j K_j)$. The properties of Definition \ref{def:base_controller} are obtained analogously:
    
    \begin{enumerate}[(i)]
        \item bounded memory: the action set sequence $\mathcal{U}_{1:T}$ has $(H, \eps)$-bounded memory, according to Theorem 5.3 with $H = O(\log(1/\eps))$ since the proof of the theorem is done separately for each $t \in [T]$.
        \item Lipschitz constant: the proxy cost $f_t^M()$ is shown to be coordinate-wise $L$-Lipschitz with respect to the parameter $M$. This holds according to Lemma 5.6 with $L = O(H)$ which is shown in this setting similarly by substituting $(A+BK)^i$ with $\prod_j^{j+i-1} (A_j+B_j K_j)$.
    \end{enumerate}
    
    We remark that Lemma 5.7 from \citet{agarwal2020boosting} also follows in an identical manner.
\end{proof}

\begin{proof}[Proof of Theorem \ref{thm:gpc_time-varying}]
As shown in the previous proof, all the building blocks for proving the regret bound in \eqref{eq:gpc_regret} apply to the time-varying dynamics setting (see Theorem 5.1 in \citet{agarwal2019online}). The essential step left is to show that the proxy loss surrogate $\tilde{f}$ function is convex w.r.t. its parameter $M^{[1:H]}$. This follows since the states and the controls are linear transformations of the parameter $M^{[1:H]}$.
\begin{lemma}\label{lemma:convex}
The loss functions $\tilde{f}_t$ are convex in the parameter $M^{[1:H]}$.
\end{lemma}
\begin{proof}
Since we assume that the cost function $c_t$ is a convex function w.r.t. its arguments, i.e. state and action, we simply need to show that $\hat{x_t}(M^{[1:H]})$ and $u_t(M^{[1:H]})$ depend linearly on $M^{[1:H]}$. The state progression is given by
    \begin{align*}
        x_{t+1} = A_t x_{t} + B_t u_t + w_t &= A_t x_{t} + B_t\left( K_t x_t + \sum_{i=1}^{H} M^{[i]} w_{t-i}\right) + w_t \\&
        = \tilde{A}_t  x_t + \left(B_t\sum_{i=1}^{H} M^{[i]} w_{t-i} + w_t\right),
    \end{align*}
    where we denote $\tilde{A}_t  = A_t + B_t K_t$.
We can therefore obtain proxy state's linearity in $M^{[1:H]}$, given by
    \begin{align*}
        \hat{x}_{t+1} &= \sum_{\tau=t-H}^{t} \left( \prod_{i=\tau+1}^{t} \tilde{A}_i \right) \cdot \left(B_{\tau} \sum_{j=1}^{H} M^{[j]} w_{\tau-j} + w_{\tau}\right) ~. 
    \end{align*}
    
The real state $x_t$ has a residual component that does not depend on $M^{[1:H]}$ at all and is also linear in this parameter. Hence, the control $u_t$ given by
    \begin{align*}
        u_t = K_t x_t + \sum_{i=1}^{H} M^{[i]} w_{t-i} ,
    \end{align*}
    is linear in $M^{[1:H]}$ as well. Thus, we have shown that $x_t$ and $u_t$ are linear transformations in $M^{[1:H]}$ and hence, the loss function $\tilde{f}_t$ is convex in $M^{[1:H]}$.
\end{proof}

The proof of Lemma \ref{lemma:convex} concludes the regret bound in \eqref{eq:gpc_regret}. Finally, we note that the action shift $\mathcal{S}_{\text{GPC}}(T)$ is that of OGD over the parameters $M^{[1:H]}_t$ which is known to be $O(\sqrt{T})$ (see Appendix \ref{sec:ogd}).
\end{proof}


\subsubsection{Proofs for Time-Varying DRC-ONS}\label{sec:drc-proof}

The results for the time-varying analogue of DRC-ONS given in Algorithm \ref{alg:DRC-ONS} hold under the assumption model of  \citet{simchowitz2020making}. In particular, suppose Assumptions 1, 3, 4 from \citet{simchowitz2020making} are satisfied with these modifications. Since we are dealing with time-varying dynamics, let Assumption 3 (and implicitly Assumption 4) holds for all $t \in [T]$, i.e. the system matrices $(A_t, B_t)$ all have bounded spectral norms and $(A+BK)^i$ for fixed systems is substituted with $\prod_{j}^{j+i-1} (A_j + B_j K_j)$ for interval $[j, j+i-1]$ for time-varying systems.

\begin{proof}[Proof of Claim \ref{claim:appl_ons}]
    The quantities mentioned above have analogous bounds given the transition from fixed to time-varying system dynamics. The properties in this claim concerning the control setup in \citet{simchowitz2020making} are better illustrated in a previous paper \citet{simchowitz2020improper} since the former mainly contributes in the OCO with memory setting, and the control to OCO with memory reduction is borrowed from the latter (see Appendix D in \citet{simchowitz2020making} as source for this statement). Hence, we prove the given claim by linking to the least recent work \citep{simchowitz2020improper} and all references henceforth are w.r.t. this paper. As an intermediate step, we note that the magnitude bounds in Lemma 5.1 still hold in the time-varying case given that only the expression for $x^{\text{nat}}$ changes not its norm bound property. The properties of Definition \ref{def:base_controller} are obtained analogously:
    
    \begin{enumerate}[(i)]
        \item bounded memory: the action set sequence $\mathcal{U}_{1:T}$ has $(H, \eps)$-bounded memory, according to Lemma 5.3 with $H = O(\log(1/\eps))$ since the proof of the lemma can be done separately for each $t > H$, and given that the given that the decay function is exponential as assumed.
        \item Lipschitz constant: the proxy cost $f_t^M()$ is shown to be coordinate-wise $L$-Lipschitz with respect to the parameter $M$. This holds according to Lemma 5.4 with $L = O(\sqrt{H})$ which is shown in this setting similarly by substituting $(A+BK)^i$ with $\prod_j^{j+i-1} (A_j+B_j K_j)$ in the Markov operator expression. Furthermore, the diameter bound for the constraint set $\mathcal{M}$ is shown in the same lemma.
    \end{enumerate}
    
    We remark that Lemmas 5.1, 5.2 from \citet{simchowitz2020improper} follow analogously.
\end{proof}

\begin{proof}[Proof of Theorem \ref{thm:drc-ons}]
As shown above, all the building blocks required to transfer the control setting to an OCO with memory setting and prove \eqref{eq:drc_regret}. The remaining component is to show that the resulting proxy cost satisfies the condition for OCO with {\it affine} memory: the loss function with memory $f_t(\cdot)$ must be given as a convex function of a variable {\it linear} in the parameter $M^{[1:H]}$ (the strong convexity of $c_t$ is necessary for exp-concavity). 

This is evident given the derivation in Lemma \ref{lemma:convex} as one can show analogously that in the time-varying dynamics setting with DRC instead of DAC, the proxy state and the action are still linear in the parameter $M^{[1:H]}$ with $x_t^{\text{nat}}$ quantities instead of the disturbances $w_t$. Therefore, we can still use Theorem 2.1 of \citet{simchowitz2020making} in the time-varying setting to bound the regret in OCO with memory by a logarithmic term in $T$. This results in the final result of \eqref{eq:drc_regret} in the time-varying dynamics setting.

Finally, we note that in the proof of Theorem 2.1, \citet{simchowitz2020making} shows that the stability gap $\mathcal{SG}_{\text{DRC-ONS}(T)}$ is polylogarithmic given Lemma 4.7 combined with equations (4.2) and (4.3). This automatically holds in our setting since the stability gap is w.r.t. parameters $M^{[1:H]}$ and the procedure Semi-ONS is unchanged for parameter update.
\end{proof}

\section{Additional Details}

\subsection{Online Parameter Choice} \label{sec:param_choice}
All the algorithms provided in this work include parameters whose choice indicates prior knowledge of the number of rounds $T$. This is a common issue in online learning algorithms, and is usually handled via the so called doubling trick. On the other hand, the first-order adaptive regret result \eqref{eq:optregret} in Theorem \ref{thm:adaoco} requires the parameter $\eta$ to depend on $OPT = \min\limits_{z \in \mathcal{K}} \sum\limits_{t=1}^T \tilde{f}_t(z)$ the best loss in hindsight. That is not a desirable assumption to make given that the surrogate losses $\tilde{f}_t$ are assumed to be picked by an adversary. In this section, we show that no prior knowledge of $OPT$ is necessary to attain \eqref{eq:optregret} as suggested by Remark \ref{remark:eta}.
\begin{proof}[Proof of Remark \ref{remark:eta}]
The statement is shown by running Algorithm \ref{alg:adaptive_reg}, MARA, with changing values of parameter $\eta$. In particular, we divide up the learning into epochs $e = 0, 1, \dots, E$ as follows: denote $OPT_e = C 4^e$, restart and run $MARA$ with $\eta_e = (4 H^2 L \shift_{\A}(T) \sqrt{OPT_e})^{-1}$ up to time $T_e = \max\{t, \text{ s.t. } \min_{i \in [N]} \sum_{\tau=1}^{t-1} \tilde{f}_{\tau}(z_{\tau}^i) + 1 \leq OPT_e  \}$. Denote $T_{-1}=0, T_E = T$, then each epoch $e \in [0, E]$ runs on time steps $[T_{e-1}+1, T_e]$. Given that $\tilde{f}_t(\cdot)$ have domain $[0, 1]$, by construction, for any interval $I_e \subseteq [T_{e-1}+1, T_e]$, we get
\begin{equation*}
    \min_{i \in [N]} \sum_{t \in I_e} \tilde{f}_t(z_t^i) \leq \min_{i \in [N]} \sum_{t=1}^{T_e} \tilde{f}_t(z_t^i) \leq \min_{i \in [N]} \sum_{t=1}^{T_e - 1} \tilde{f}_t(z_t^i) + 1 \leq OPT_e,
\end{equation*}
for all $e \in [0, E]$. Therefore, given the choice of $\eta_e$, we use the bound above, and Theorem \ref{thm:meta} for each epoch $e$, as well as $\shift_{\A}(T) = O(\log T), \reg_{\A}(T) = O(\log T)$ for simplicity, to get for any interval $I_e \subseteq [T_{e-1}+1, T_e]$,
\begin{equation*}
    \sum_{t \in I_e} \E[f_t(z_{t-H:t})] \leq \min_{i \in [N]} \sum_{t \in I_e} \tilde{f}_t(z_t^i) + \tilde{O}(H^2 L \cdot \sqrt{OPT_e}) ~.
\end{equation*}
On the other hand, the number of epochs can be implcitily bounded as follows,
\begin{equation*}
    C 4^{E-1} = OPT_{E-1} < \min_{i \in [N]} \sum_{t=1}^{T_{E-1}} \tilde{f}_t(z_t^i) + 1 \leq \min_{i \in [N]} \sum_{t=1}^{T} \tilde{f}_t(z_t^i) + 1 \leq OPT + 2 \reg_{\A}(T),
\end{equation*}
where we assume $\reg_{\A}(T) \geq 1$ for simplicity. This bound implies that 
\begin{equation*}
    \sum_{e=0}^E \sqrt{OPT_e} = \sqrt{C} \sum_{e=0} 2^e \leq \sqrt{C} 2^{E+1} \leq 4 \sqrt{OPT + 2 \reg_{\A}(T)}
\end{equation*}
Finally, we can conclude that for any interval $I \subseteq [T]$, dividing it into subintervals $I_e, e \in [0, E]$ to attain
\begin{align*}
    \sum_{t \in I} \E[f_t(z_{t-H:t})] &\leq \sum_{e=0}^E \left( \min_{i \in [N]} \sum_{t \in I_e} \tilde{f}_t(z_t^i) + \tilde{O}(H^2 L \cdot \sqrt{OPT_e}) \right) \\
    &\leq \min_{i \in [N]} \sum_{t \in I} \tilde{f}_t(z_t^i) + \tilde{O}(H^2 L \cdot \sqrt{OPT + 2 \reg_{\A}(T)}) \\
    &\leq \min_{z \in \mathcal{K}} \sum_{t \in I} \tilde{f}_t(z) + \tilde{O}(H^2 L \cdot \sqrt{OPT}),
\end{align*}
obtained with no additional complexity or regret overhead asymptotically.
\end{proof}

\subsection{Facts about projected OGD}\label{sec:ogd}
We make use of Facts \ref{claim:scogd}, \ref{claim:cogd} about the projected OGD algorithm. Even though these are well-known facts, we formally state and provide their proofs for completeness and consistency with our notation.

\paragraph{Projected Online Gradient Descent.} This is the most common algorithm in online optimization, given by the following update rule:
\begin{equation*}
    \forall t = 1, 2, \dots, \quad z_{t+1} = \Pi_{\mathcal{K}} \left[ z_t - \eta_t \nabla_t \right],
\end{equation*}
where $\nabla_t = \nabla f_t(z_t)$ and $\Pi_{\mathcal{K}}[\cdot]$ denotes the projection operation onto the constraint set $\mathcal{K}$ that has diameter $D$. The losses $f_t(\cdot)$ have gradients with bounded norms, i.e. $\| \nabla f_t(z) \| \leq G, \forall z \in \mathcal{K}$.
\begin{fact}\label{claim:scogd}
Let $\A_{\text{scOGD}}$ be the projected OGD algorithm over $\alpha$-strongly convex loss functions with stepsizes $\eta_t = \frac{1}{\alpha t}$, then $\reg_{\A_{\text{scOGD}}}(T)  = O(\frac{G^2}{\alpha} \log T)$ and $\shift_{\A_{\text{scOGD}}}(T)  = O(\frac{G}{\alpha} \log T)$.
\end{fact}

\begin{proof}[Proof of Fact \ref{claim:scogd}]
    First, we show the bound on the action shift $\shift_{\A_{\text{scOGD}}}(T)$. With the constraint set $\mathcal{K}$ being convex, one has that for all $t \in [T]$,
    \begin{equation*}
        \| z_{t+1} - z_t \| \leq \| - \eta_t \nabla_t \| \leq \eta_t \cdot G ~.
    \end{equation*}
    Summing this bound up for all $t \in [T-1]$, we get the desired bound $\shift_{\A_{\text{scOGD}}}(T) = O(\frac{G}{\alpha} \log T)$ since $\eta_t = \frac{1}{\alpha t}$. On the other hand, showing regret follows standard methodology. Denote $z^* = \argmin_{z \in \mathcal{K}} \sum_{t=1}^T f_t(z)$ and use strong convexity to get $f_t(z_t) - f_t(z^*) \leq \nabla_t^{\top}(z_t-z^*) - \alpha/2 \cdot \| z_t-z^*\|^2$. Use the fact that a projected point is closer to $z^*$ than the point being projected along with algebraic manipulations to get $\nabla_t^{\top}(z_t-z^*) \leq \frac{1}{2 \eta_t} (\| z_t-z^*\|^2 - \| z_{t+1}-z^*\|^2) + \frac{\eta_t}{2} G^2$. Putting this all together with the stepsize values $\eta_t = \frac{1}{\alpha t}$ results in regret being bounded by $\frac{G^2}{2}  \cdot \sum_{t=1}^T \eta_t$ which yields $\reg_{\A_{\text{scOGD}}}(T) = O(\frac{G^2}{\alpha} \log T)$.
\end{proof}

\begin{fact}\label{claim:cogd}
Let $\A_{\text{cOGD}}$ be the projected OGD algorithm over general convex loss functions with stepsizes $\eta_t = \frac{D}{G \sqrt{t}}$, then $\adreg_{\A_{\text{cOGD}}}(T)  = O(D G \sqrt{T})$ and $\shift_{\A_{\text{cOGD}}}(T)  = O(D \sqrt{T})$.
\end{fact}

\begin{proof}[Proof of Fact \ref{claim:cogd}]
    The action shift bound is achieved in a similar straightforward fashion. In particular, with stepsizes $\eta_t = \frac{D}{G \sqrt{T}}$, we have
    \begin{equation*}
        \sum_{t=1}^{T-1} \| z_{t+1}-z_t \| \leq G \sum_{t=1}^{T-1} \eta_t = D \sum_{t=1}^{T-1} \frac{1}{\sqrt{t}},
    \end{equation*}
    which results in $\shift_{\A_{\text{cOGD}}}(T) = O(D \sqrt{T})$. The adaptive regret is shown using the same building blocks as in the proof of Fact \ref{claim:scogd}: use convexity to get $f_t(z_t) - f_t(z^*) \leq \nabla_t^{\top} (z_t-z^*)$; the bound $\nabla_t^{\top}(z_t-z^*) \leq \frac{1}{2 \eta_t} (\| z_t-z^*\|^2 - \| z_{t+1}-z^*\|^2) + \frac{\eta_t}{2} G^2$ holds in this case as well. For any interval $[r, s] = I \subseteq [T]$, consider $z^* = \argmin_{z \in \mathcal{K}} \sum_{t=r}^s f_t(z)$ and sum up the given bounds from $t=r$ to $t=s$ to get
    \begin{align*}
        \sum_{t=r}^s f_t(z_t) - \min_{z \in \mathcal{K}} \sum_{t=r}^s f_t(z) &= \sum_{t=r}^s \left[ f_t(z_t) - f_t(z^*) \right] \leq \\ &\leq \| z_r - z^*\|^2 \cdot \frac{1}{2 \eta_r} + \sum_{t=r+1}^s \| z_t - z^* \|^2 \cdot \left( \frac{1}{2 \eta_t} - \frac{1}{2 \eta_{t-1}} \right) - \\
        &- \| z_{s+1}-z^*\|^2 \cdot \frac{1}{2 \eta_s} + \frac{G^2}{2} \cdot \sum_{t=r}^s \eta_t \leq D^2 \cdot \frac{1}{2 \eta_s} + \frac{G^2}{2} \cdot \sum_{t=r}^s \eta_t
    \end{align*}
    Given the values of the stepsizes $\eta_t = \frac{D}{G \sqrt{t}}$, we conclude that $\adreg_{\A_{\text{cOGD}}}(T) = O(D G \sqrt{T})$.
\end{proof}

Finally, we remark that Fact \ref{claim:cogd} applied on \eqref{eq:memadregret2} from Lemma \ref{lem:memadregret} implies $O(\sqrt{T})$ adaptive regret for functions with memory achieved by $\A_{\text{cOGD}}$ on the surrogate losses.

\subsection{Working Set Construction}\label{sec:working}
Section \ref{sec:efficient} describes a way to efficiently implement the main algorithm in this work. This efficient implementation makes use of the working sets $\{S_t\}_{t \in [T]}$ along with its properties in Claim \ref{claim:workingsets}. In this section, we show the explicit construction of these working sets as in \citet{hazan2009efficient} and prove Claim \ref{claim:workingsets}.

For any $i \in [T]$, let it be given as $i = r 2^k$ with $r$ odd and $k$ nonnegative. Denote $m = 2^{k+2}+1$, then $i \in S_t$ if and only if $t \in [i, i+m]$. This fully describes the construction of the working sets $\{S_t\}_{t \in [T]}$, and we proceed to prove its properties.

\begin{proof}[Proof of Claim \ref{claim:workingsets}]
    For all $t \in [T]$ we show the following properties of the working sets $S_t$. 
    
    (i) $|S_t| = O(\log T)$: if $i \in S_t$ then $1 \leq i = r 2^k \leq t$ which implies that $0 \leq k \leq \log_2{t}$. For each fixed $k$ in this range, if $r 2^k = i \in S_t$ then $i \in [t-2^{k+2}-1, t]$ by construction. Since $[t-2^{k+2}-1, t]$ is an interval of length $2^{k+2}+2 = 4 \cdot 2^k + 2$, it can include at most $3$ numbers of the form $r 2^k$ with $r$ odd. Thus, there is at most $3$ numbers $i = r 2^k \in S_t$ for each $0 \leq k \leq \log_2{t}$ which means that $|S_t| = O(\log t) = O(\log T)$.
    
    (ii) $[s, (s+t)/2] \cap S_t \neq \emptyset$ for all $s \in [t]$: this trivially holds for $s=t-1, t$. Let $2^l \leq (t-s)/2$ be the largest such exponent of $2$. Since the size of the interval $[s, (s+t)/2]$ is $\lfloor (t-s)/2 \rfloor$, then there exists $u \in [s, (s+t)/2]$ that divides $2^l$. This means that the corresponding $m \geq 2^{l+2}+1 > t-s$ for $u \geq s$ is large enough so that $t \in [u, u+m]$, and consequently, $u \in S_t$.
    
    (iii) $S_{t+1} \backslash S_t = \{t+1\}$: let $i \in S_{t+1}$ and $i \not \in S_t$, which is equivalent to $t+1 \in [i, i+m]$ and $t \not \in [i, i+m]$. Clearly, $i = t+1$ satisfies these conditions and is the only such number.
    
    (iv) $|S_t \backslash S_{t+1}| \leq 1$: suppose there exist two $i_1, i_2 \in S_t \backslash S_{t+1}$. This implies that $i_1 + m_1 = t = i_2 + m_2$ which in turn means $2^{k_1} (r_1+4) = 2^{k_2} (r_2+4)$. Since both $r_1+4, r_2+4$ are odd, then $k_1 = k_2$, and consequently, $r_1=r_2$ resulting in $i_1=i_2$. Thus, there can not exist two different members of $S_t \backslash S_{t+1}$ which concludes that $|S_t \backslash S_{t+1}| \leq 1$.
\end{proof}

\subsection{Supplemental Experimental Details}

\paragraph{Implementation.}For MARC, we implement a modified efficient version of Algorithm \ref{alg:adaptive_control_main} (described in section \ref{sec:efficient}) using GPC as the baseline controller. The changes we make for experimental purposes are: (i) we play the expert with the greatest weight at the given timestep (rather than sample), (ii) we instantiate a new learner every $20$ timesteps (rather than every timestep), and (iii) we pad the liftetimes so that they are at least $100$ timesteps. For iLQR, we improve the classical algorithm using the techniques described in \cite{iLQR}, namely (1) optimizing over regularization terms and (2) introducing a backtracking line search parameter that protects against divergence. 

\paragraph{Training.} For MARC, every learner is initialized arbitrarily when instantiated by the meta-algorithm (as described in section \ref{sec:efficient} and above). After every action taken, MARC receives the new state and the linearized dynamics $(A_t, B_t)$\footnote{These are obtained via automatic differentation using JAX \cite{jax2018github}.} and: 1) updates the weight of each learner according to their respective hypothetical surrogate proxy cost (Line 12 of Algorithm \ref{alg:adaptive_control_main}) 2) feeds them to the learners which use them to construct the proxy loss function and execute online gradient descent (see Algorithm \ref{alg:GPC}). We focus on the task of online control and hence there is no learning across different episodes. For iLQR, the algorithm has knowledge of starting state and the state evolution function in advance and uses it to plan. However, after that it only executes the computed plan and does not adjust based on new observed states (hence the lack of robustness to noise).

\paragraph{Hyperparameters.} We select $3$ ``tuning'' seeds (which determine $3$ different starting states) on which we run all the algorithm to tune the (few) hyperparameters. For GPC, we take $H=10$ and only tune the learning rate, choosing the best out of $[0.001, 0.005, 0.01, 0.05, 0.1]$ (which is $0.01$). For MARC, the only two tunable hyperparameters are $\eta$ and $\sigma$ (since we feed the already individually tuned GPC formulation): 1) for $\eta$ we choose the best out of $[0.001, 0.005, 0.01, 0.05, 0.1]$ (which is $0.05$), 2) for $\sigma$ we choose the best out of $[10^{-3}, 10^{-2}, 10^{-1}]$ (which is $10^{-2}$). For iLQR, we set the tolerance to $10^{-16}$ and select a \# of trajectory improvement iterations high enough so that the trajectory converges. Due to computational constraints, we only test iteration numbers $[10, 20, 30, 40, 50, 100, 200]$ and since we observe no improvement for the higher numbers we settle on $10$.

\paragraph{Results.} We then test all the tuned algorithm on $3$ different ``testing'' seeds (on which the online controllers learn from scratch, but with the already tuned hyperparameters) and report the mean results with confidence intervals in Figure \ref{pendulum}.

\end{document}